\documentclass{article}

\usepackage{fullpage}
\date{}

\usepackage[style=numeric,maxbibnames=99,maxcitenames=2,natbib=true,backend=bibtex8]{biblatex}

\usepackage[utf8]{inputenc} %
\usepackage[T1]{fontenc}    %
\usepackage{hyperref}       %
\usepackage{url}            %
\usepackage{booktabs}       %
\usepackage{amsfonts}       %
\usepackage{nicefrac}       %
\usepackage{microtype}      %

 \usepackage{graphicx}
 \usepackage{lmodern}

\usepackage{subcaption}
\usepackage{booktabs}
\usepackage{xcolor}
\usepackage{mathtools,cuted}

\usepackage{amsmath,amsthm,amssymb}
\usepackage{breqn}
\usepackage{forloop}
\usepackage{algorithm}
\usepackage{algorithmicx}
\usepackage{algpseudocode}
\usepackage{paralist}
\usepackage{xspace}
\usepackage[capitalize]{cleveref}
\usepackage[title]{appendix}
\crefname{appsec}{Appendix}{Appendices}

\newcommand{\kl}[2]{D_{\mathrm{KL}}( #1 \parallel #2 )}
\newcommand{\js}[2]{D_{\mathrm{JS}}( #1 \parallel #2 )}
\newcommand{\wjs}[3]{D_{\mathrm{GJS}}^{#1}( #2 \parallel #3 )}
\newcommand{\td}[2]{ \Delta ( #1 \parallel #2 )}

\newtheorem{theorem}{Theorem}
\newtheorem{proposition}{Proposition}

\newtheorem{lemma}{Lemma}
\theoremstyle{definition}
\newtheorem{definition}{Definition}

\DeclareMathOperator{\hel}{hel}
\DeclareMathOperator{\pr}{Pr}

\DeclareMathOperator{\sech}{sech}

\DeclareMathOperator{\mil}{mil}
\DeclareMathOperator{\sign}{sign}

\newcommand{\ie}{\emph{i.e.}}
\newcommand{\eg}{\emph{e.g.}}

\newcommand{\kr}{Kre\u{\i}n\xspace}
\newcommand{\splsh}{\textsc{Simple-LSH}\xspace}

\newcommand{\defcal}[1]{\expandafter\newcommand\csname 
c#1\endcsname{{\mathcal{#1}}}}
\newcommand{\defbb}[1]{\expandafter\newcommand\csname 
b#1\endcsname{{\mathbb{#1}}}}
\newcommand{\defbf}[1]{\expandafter\newcommand\csname 
bf#1\endcsname{{\mathbf{#1}}}}
\newcounter{calBbCounter}
\forLoop{1}{26}{calBbCounter}{
	\edef\letter{\Alph{calBbCounter}}
	\expandafter\defcal\letter
	\expandafter\defbb\letter
	\expandafter\defbf\letter
}
\forLoop{1}{26}{calBbCounter}{
	\edef\letter{\alph{calBbCounter}}
	\expandafter\defbf\letter
}

\title{Locality-Sensitive Hashing for $f$-Divergences and \kr Kernels: Mutual 
Information Loss and Beyond}

\author{ %
	Lin~Chen$ ^{1,2} $ \quad Hossein~Esfandiari$ ^2 $ \quad Thomas~Fu$ ^2 $ 
	\quad Vahab~S.~Mirrokni$ ^2 $  \\
	$ ^1 $Yale University\quad 
	$ ^2 $Google Research\\
	\texttt{lin.chen@yale.edu, 
		\{esfandiari,thomasfu,mirrokni\}@google.com}
}

\begin{document}

\maketitle

\begin{abstract}
	Computing approximate nearest neighbors in high dimensional spaces is a central
	problem in large-scale data mining with a wide range of applications in machine learning
	and data science. A popular and effective technique in computing nearest 
	neighbors approximately is the \emph{locality-sensitive hashing} (LSH) 
	scheme. %
	In this paper, we aim to develop LSH schemes for distance functions that 
	measure the distance between two probability distributions, particularly for $ f $-divergences as well as a generalization to capture mutual information loss. %
	First, we provide a general framework to design LHS schemes for 
	$f$-divergence distance functions and develop LSH schemes for the 
	\emph{generalized Jensen-Shannon divergence} and \emph{triangular 
	discrimination} in this framework. We show a two-sided approximation result 
	for approximation of the generalized Jensen-Shannon divergence by the 
	Hellinger distance, which may be of independent interest.
	Next, we show a general method of reducing the problem of designing an LSH 
	scheme for a \kr kernel (which can be expressed as the difference of two 
	positive definite kernels) to the problem of maximum inner product search. 
	We exemplify this method by applying it to the \emph{mutual information 
	loss}, due to its several important applications such as model 
	compression.

\end{abstract}

\section{Introduction}
A central problem in machine learning and data mining is to find top-$k$ 
similar items to each item in a dataset. Such problems,  referred to as {\em 
approximate nearest neighbor} problems,  are especially challenging in high 
dimensional spaces and are an important part of a wide range of data mining 
tasks such as finding near-duplicate pages in a corpus of images or web pages, 
or clustering items in a high-dimensional metric space. 
A popular technique for solving these problems is the \emph{locality-sensitive 
hashing} (LSH) technique~\citep{indyk1998approximate}. In this method, items in 
a high-dimensional metric space are first mapped into buckets (via a hashing 
scheme) with the property that closer items have a higher chance of being 
assigned to the same bucket. LSH-based nearest neighbor methods 
limit their scope of search to the items that fall into the same bucket in 
which the target item resides~\footnote{We note that LSH is a popular 
data-independent technique for nearest neighbor search. Another category of 
nearest neighbor search algorithms, referred to as data-dependent techniques, are 
\emph{learning-to-hash} methods~\citep{wang2018survey} which learn a hash 
function that maps each item to a compact code. However, this line of work is out of the scope of this paper.}.

Locality sensitive hashing was first introduced and studied by \cite{indyk1998approximate}.
They provide a family of basic locality-sensitive hash functions for the 
Hamming distance in a $d$-dimensional space and  for the $L^1$ distance in a  
$d$-dimensional Euclidean space.
They also show that such a family of hash functions provides a randomized $(1 + 
\epsilon)$-approximation algorithm for the nearest neighbor search problem with 
sublinear space and sublinear query time.
Following \cite{indyk1998approximate}, several families of locality-sensitive 
hash functions have been designed and implemented for different metrics, each 
serving a certain application. We summarize further results in this area in 
\cref{sec:related}.

In several applications, data points can be represented as probability distributions. One example is the space of users' browsed web pages, read articles or watched videos. In order to represent such data, one can represent each user by a distribution of documents they read, and the documents by topics included in those documents. Other examples are time series distributions, content of documents, or images that can be represented as histograms. Particularly, analysis of similarities in time series distributions or documents can be used in the context of attacks, and spam detection. Analysis of user similarities can be used in recommendation systems and online advertisement. 

In fact, many of the aforementioned applications deal with huge datasets and 
require very time efficient algorithms to find similar data points. These 
applications motivated us to study LSH functions for distributions, especially 
for distance measures with information-theoretic justifications. In fact, in 
addition to $k$-nearest neighbor, LSH functions can be used to implement very 
fast distributed algorithms for traditional clusterings such as $k$-means 
\cite{bhaskara2018distributed}.

Recently, \citet{mao2017s2jsd} noticed the importance and lack of LSH functions 
for the distance of distributions, especially for information-theoretic 
measures. They attempted to design an LSH to capture the famous 
\emph{Jensen-Shannon (JS) divergence}. However, instead of directly providing 
locality-sensitive hash functions for Jensen-Shannon divergence, they take two 
steps to turn this distance function into a new distance function that is 
easier to hash. They first looked at a less common divergence measure 
\emph{S2JSD} which is the square root of two times the JS  
divergence. 
Then they defined a related distance function $S2JSD_{new}^{approx}$, which 
was  
obtained by only keeping the linear 
terms in the Taylor expansion of the logarithm in the expression of S2JSD %
and designed a locality-sensitive hash function for the new measure 
$S2JSD_{new}^{approx}$. This is an interesting work; however, unfortunately it 
does not provide any bound on the actual JS divergence using the LSH that they 
designed for $S2JSD_{new}^{approx}$. Our results resolve this issue by 
providing LSH schemes with provable guarantees for information-theoretic 
distance measures including the JS divergence and its generalizations.

\citet{mu2010non} proposed an LSH algorithm for non-metric data by embedding 
them into a reproducing kernel \kr space. However, their method is indeed 
data-dependent. Given a finite set of data points $ \cM $, they compute the 
distance matrix $ D $ whose $ (i,j) $-entry is the distance between $ i $ and $ 
j $, where both $ i $ and $ j $ are data points in $ \cM $. Data is embedded 
into a reproducing kernel \kr space by performing singular value decomposition 
on a transform of the distance matrix $ D $. The embedding changes if we are 
given another dataset.

\textbf{Our Contributions.}
In this paper, we first study LSH schemes for $ f $-divergences\footnote{The 
formal definition of $ f $-divergence is presented in \cref{sub:f_divergence}.} 
between 
 two probability distributions. 
We first in \cref{thm:main} provide a simple reduction tool for designing LSH 
schemes for 
the family of $f$-divergence distance functions. This proposition is not hard 
to prove but might be of independent interest. 
Next we use this tool and provide LSH schemes for two examples of 
$f$-divergence distance functions, Jensen-Shannon divergence and 
triangular discrimination.
Interestingly our result holds for a generalized version of Jensen-Shannon divergence.
We apply this tool to design and analyze an LSH scheme for the 
generalized Jensen-Shannon (GJS)
 divergence through approximation by the squared Hellinger distance. 
We use a similar technique to provide an LSH for triangular discrimination. Our 
approximation is provably lower bounded by a factor $0.69$ for the 
Jensen-Shannon
 divergence and is lower bounded by a factor $0.5$ for triangular 
 discrimination. 
The approximation result of the generalized Jensen-Shannon divergence by the 
squared Hellinger requires a more involved analysis and the lower and upper 
bounds depend on the weight parameter.
This approximation result may be of independent interest for other machine learning tasks such as approximate information-theoretic clustering~\citep{chaudhuri2008finding}.
Our technique may be useful for designing LSH schemes for other $f$-divergences.

Next, we 
 propose a general approach to designing an LSH for \kr kernels. A \kr kernel is a kernel function that can be expressed as the difference of two positive definite kernels. 
Our approach is built upon a reduction to the problem of maximum inner product search (MIPS)~\citep{shrivastava2014asymmetric,neyshabur2015symmetric,yan2018norm}. 
In contrast to our LSH schemes for $ f $-divergence functions via approximation, our approach for \kr kernels involves no approximation and is theoretically \emph{lossless}.
Contrary to \citep{mu2010non}, this approach is data-independent.
We exemplify our approach by designing an LSH function specifically for mutual information loss. Mutual information loss is of our particular interest due to its several important applications such as model compression \cite{bateni2019categorical,dhillon2003divisive}, and compression in discrete memoryless channels 
\cite{
	kartowsky2018greedy,sakai2014suboptimal,
	zhang2016low}.

\subsection{Other Related Work}\label{sec:related}
\citet{datar2004locality} designed an LSH for $ L^p $ distances using 
$ p $-stable distributions.
\citet{broder1997resemblance} designed MinHash for the Jaccard similarity.
 LSH for other distances and similarity measures were proposed later, for 
example, angle similarity 
\citep{charikar2002similarity}, spherical LSH on a unit hypersphere 
\citep{terasawa2007spherical}, rank similarity~\citep{yagnik2011power}, and 
non-metric LSH~\citep{mu2010non}.
\citet{li2014coding} demonstrated that uniform quantization outperforms the 
standard method in \citep{datar2004locality} with a random offset.
\citet{gorisse2012locality} proposed an LSH family for $ \chi^2 $ distance by 
relating it to the $ L^2 $ distance via an algebraic transform. 
Interested readers are referred to a more comprehensive survey of existing LSH methods \citep{wang2014hashing}.
Another related problem is the construction of feature maps of positive 
definite kernels. A feature map maps a data point into a usually 
higher-dimensional space such that the inner product in that space agrees with 
the kernel in the original space. Explicit feature maps for additive kernels 
are introduced in~\citep{vedaldi2012efficient}.
Bregman divergences are another broad class of distances that arise naturally 
in practical applications. The nearest neighbor search problem for Bregman 
divergences were studied in 
\citep{abdullah2012approximate,abdullah2015directed,abdelkader2019approximate}.

\section{Preliminaries}

\subsection{Locality-Sensitive Hashing}\label{sub:prelim_lsh}
Let $ \mathcal{M} $ be the universal set of items (the database), endowed with 
a distance function $ D $.
Ideally, we would like to have a family of hash functions such that for any two 
items $ p $ and $ q $ in $ \mathcal{M} $ that are close to each other, their 
hash values collide with a higher probability, and if they reside far apart, 
their hash values collide with a lower probability. A family of hash functions 
with the above property is said to be locality-sensitive.
A hash value is also known as a bucket in other literature.  Using this 
metaphor, hash functions are imagined as sorters that place items into buckets.
If hash functions are locality-sensitive, it suffices to search the bucket into 
which an item falls if one wants to know its nearest neighbors.
 The  $ (r_1,r_2,p_1,p_2) $-sensitive LSH family formulates the intuition of 
 locality sensitivity and is formally
defined in 
\cref{def:LSH}.

\begin{definition}[\cite{indyk1998approximate}]\label{def:LSH}
	Let $ \mathcal{H} = \{ h:\mathcal{M}\to U \} $ be a 
	family of hash functions, where $ U $ is the set of possible hash values. 
	Assume that there is a distribution $ h\sim 
	\mathcal{H} $ over the family of 
	functions. This family $ \mathcal{H} $ is called $ (r_1, r_2, p_1,p_2) 
	$-sensitive  ($ r_1<r_2 $ and $ p_1>p_2 $) for $ D $, if for $ \forall p, q 
	\in \mathcal{M} $ the following statements hold:
	(1) if $ D(p, q) \le r_1 $, then $ \pr_{h\sim 
			\mathcal{H}}[h(p)=h(q)]\ge p_1 $; (2)
	if $ D(p, q) > r_2 $, then $ \pr_{h\sim 
			\mathcal{H}}[h(p)=h(q)]\le p_2 $.
\end{definition}

We would like to note that the gap between the high probability $ p_1 $ and $ 
p_2 $ can be amplified by constructing a compound hash function that 
concatenates multiple functions from an LSH family. For example, one can 
construct $ g:\mathcal{M}\to U^K $ such that $ g(p)\triangleq (h_1(p), \dots, 
h_K(p)) $ for $ \forall p\in \mathcal{M} $, where $ h_1,\dots, h_K $ are chosen 
from the LSH family $ \mathcal{H} $. This conjunctive construction reduces the 
amount of items in one bucket. To improve the recall, an additional disjunction 
is introduced. To be precise, if $ g_1,\dots, g_L $ are $ L $ such compound 
hash functions, we search all of the buckets $ g_1(p), \dots, g_L(p) $ in order 
to 
find the nearest neighbors of $ p $.

\subsection{$ f $-Divergence}\label{sub:f_divergence}

Let $ P $ and $ Q $ be two probability measures associated with a common sample 
space $ \Omega $. We write $ P\ll Q$ if $ P $ is absolutely continuous with 
respect to $ Q $, which requires that for every subset $ A $ of $ \Omega $, $ 
Q(A)=0 $ imply $ P(A)=0 $. 

Let $ f:(0,\infty)\to \mathbb{R} $ be a convex function that satisfies $ f(1)=0 
$. If $ P\ll Q $, the 
$ f $-divergence from $ P $ to $ Q 
$~\citep{csiszar1964informationstheoretische} 
is defined by
\begin{equation}\label{eq:definition_f_divergence}
D_f(P\parallel Q) = \int_{\Omega} f\left( \frac{d P}{d Q} \right) 
d Q,
\end{equation}
provided that the right-hand side exists,
where $ \frac{d P}{d Q} $ is the Radon-Nikodym derivative of $ P $ with 
respect to $ Q $.
In 
general, an $ f $-divergence is not symmetric: $ D_f(P\parallel Q)\ne 
D_f(Q\parallel P) $.

If $ f_{\mathrm{KL}}(t) = t\ln t + (1-t) $, the $ f_{\mathrm{KL}} $-divergence 
yields the \emph{KL divergence}
$
\kl{P}{Q} = \int_{\Omega} \ln\frac{d P}{d Q} d P$~\cite{cover2012elements}.
If $ \hel(t)=\frac{1}{2}(\sqrt{t}-1)^2 $,
 the $ \hel 
$-divergence is the 
\emph{squared Hellinger distance}
$
H^2(P,Q) = \frac{1}{2} \int_{\Omega} 
(\sqrt{dP}-\sqrt{dQ})^2$~\cite{daskalakis2017square}.
If $ \delta(t) = \frac{(t-1)^2}{t+1} $, the $ \delta $-divergence is the 
\emph{triangular discrimination} (also known as Vincze-Le Cam 
distance)~\citep{le2012asymptotic,vincze1981concept}. If the sample space is 
finite, the triangular discrimination between $ P $ and $ Q $ is given by 
$\Delta(P\parallel Q) = \sum_{i\in \Omega} \frac{(P(i)-Q(i))^2}{P(i)+Q(i)}$.

The \emph{Jensen-Shannon (JS) divergence} is a symmetrized version of the 
KL divergence. If $ P\ll Q $, $ Q\ll P $ and $ M=(P+Q)/2 $, the JS divergence 
is defined by
\begin{equation}\label{eq:jensen-shannon}
\js{P}{Q} = \frac{1}{2} \kl{P}{M} + \frac{1}{2} \kl{Q}{M}\enspace.
 \end{equation}

\subsection{Mutual Information Loss and Generalized Jensen-Shannon Divergence}
\label{sub:mil-gjs}

The mutual information loss arises naturally in many machine learning tasks, 
such as information-theoretic clustering~\citep{dhillon2003divisive} and 
categorical feature compression~\citep{bateni2019categorical}. 

Suppose that two random variables $ X $ and $ C $ obeys a joint distribution $ p(X,C) $. This joint distribution can model a dataset where $ X $ denotes the feature value of a data point and $ C $ denotes its label~\citep{bateni2019categorical}.
Let $\mathcal{X}$ and $\mathcal{C}$ denote the support of $X$ and $C$ (\ie, the 
universal set of all possible feature values and labels), respectively. 
Consider clustering two feature values into a new combined value. This 
operation can be represented by the following map 
 \[ 
\pi_{x,y}: \mathcal{X}\to \mathcal{X}\setminus\{x,y\}\cup \{z\}\quad\text{such that}\quad \pi_{x,y}(t) = \begin{cases}
t, & t\in \cX\setminus \{x, y\}\enspace,\\
z, & t = x,y\enspace,
\end{cases}
 \]
where $x$ and $y$ are the two feature values to be clustered and $z\notin \cX$ 
is the new combined feature value. To make the dataset after applying the map 
$\pi_{x,y}$ preserve as much information of the original dataset as possible, 
one has to select two feature values $x$ and $y$ such that the mutual 
information loss incurred by the clustering operation
$
\mil(x,y) = I(X;C) - I(\pi_{x,y}(X);C)
$
 is minimized,
where $I(\cdot;\cdot)$ is the mutual information between two random 
variables~\cite{cover2012elements}. 
Note that the \emph{mutual information loss (MIL) divergence} $ 
\mil:\cX\times\cX\to\bR $ is symmetric in both arguments and always 
non-negative due to the data processing inequality~\citep{cover2012elements}.

Next, we motivate the generalized Jensen-Shannon divergence. If we let $P$ and 
$Q$ be the conditional distribution of $C$ given $X=x$ and $X=y$, respectively, 
such that $ P(c) = p(C=c| X=x) $ and $ Q(c)=p(C=c| X=y) $, the mutual 
information loss can be re-written as
\begin{equation}\label{eq:mil}
    \lambda \kl{P}{M_\lambda} + (1-\lambda) \kl{Q}{M_\lambda}\enspace,
\end{equation}
where  $\lambda=\frac{p(x)}{p(x)+p(y)}$ and the distribution $ 
M_\lambda=\lambda P + (1-\lambda)Q $.
Note that \eqref{eq:mil} is a generalized version of \eqref{eq:jensen-shannon}. 
Therefore, we define the 
\emph{generalized Jensen-Shannon (GJS) 
divergence} between $P$ and $Q$~\citep{lin1991divergence,ali1966general,dhillon2003divisive}
by
$
\wjs{\lambda}{P}{Q} = \lambda \kl{P}{M_\lambda} + (1-\lambda) 
\kl{Q}{M_\lambda}$,
where $ \lambda\in [0,1] $ and $ M_\lambda=\lambda P + (1-\lambda)Q $. We 
immediately have $ \wjs{1/2}{P}{Q} = \js{P}{Q} $, which indicates that the 
JS divergence is indeed a special case of the GJS divergence when $ 
\lambda=1/2 $. The GJS
 divergence has 
another equivalent definition
$
	\wjs{\lambda}{P}{Q} = H(M_\lambda) - \lambda H(P) - 
	(1-\lambda) H(Q)$,
where $H(\cdot)$ denotes the Shannon entropy~\cite{cover2012elements}. In 
contrast to the MIL 
divergence, the GJS $ \wjs{\lambda}{\cdot}{\cdot} $ is not symmetric in general 
as the weight $ \lambda\in [0,1] $ is fixed and not necessarily equal to $ 1/2 
$. We will show in \cref{lem:wjs-as-f-divergence} that the GJS divergence is an 
$ f $-divergence.

\subsection{Positive Definite Kernel and \kr Kernel}

 We first review the definition of a positive definite kernel.

\begin{definition}[Positive definite kernel~\citep{scholkopf2001kernel}]
	Let $ \cX $ be a non-empty set. A symmetric, real-valued map $ 
	k:\cX\times\cX\to \bR $ is a positive definite kernel on $ \cX $ if for all 
	positive integer $ n $, real numbers $ a_1,\dots,a_n\in \bR $, and $ 
	x,\dots,x_n\in \cX $, it holds that $
	\sum_{i=1}^n \sum_{j=1}^n a_i a_j k(x_i,x_j) \ge 0$.
\end{definition}

A kernel is said to be a \emph{\kr} kernel if it can be represented as 
the difference of two positive definite kernels. The formal definition is presented below.
\begin{definition}[\kr kernel~\citep{ong2004learning}]
	Let $ \cX $ be a non-empty set. A symmetric, real-valued map $ k:\cX\times\cX\to \bR $ is a \kr kernel on $ \cX $ if there exists two positive definite kernels $ k_1 $ and $ k_2 $ on $ \cX $ such that $ k(x,y)=k_1(x,y)-k_2(x,y) $ holds for all $ x,y\in \cX $.
\end{definition}
 
\section{LSH Schemes for $ f $-Divergences}
We build LSH schemes for $ f $-divergences based on approximation via another $ 
f $-divergence if the latter admits an LSH family. If $ D_f $ and $ D_g $ are 
two divergences associated with convex functions $ f $ and $ g $ as defined by 
\eqref{eq:definition_f_divergence}, the approximation ratio of $ 
D_f(P\parallel Q) $ to $ D_g(P \parallel Q) $ is determined by the ratio of 
the functions $ f $ and $ g $, as well as the ratio of $ P $ to $ Q $ (to 
be precise, $ \inf_{i\in \Omega} \frac{P(i)}{Q(i)} $)~\citep{sason2016f}.

\begin{proposition}[{Proof in \cref{app:main}}]\label{thm:main}
	Let $ \beta_0\in (0,1),L,U>0 $ and let $ f $ and $ g $ be two convex 
	functions $ 
	(0,\infty)\to \mathbb{R} $ 
	that obey $ f(1)=0 $, $ g(1)=0 $, and $f(t), g(t)>0 $ for every $ t\ne 1 
	$. 
	Let $ \mathcal{P} $ be a set of probability measures
 on a  finite sample space $ \Omega $ such that for every $ i\in \Omega $ and $ 
 P, Q\in \mathcal{P} $, $
 0< \beta_0\le \frac{P(i)}{Q(i)} \le \beta_0^{-1}
 $.
	 Assume that for every $ \beta\in 
	 (\beta_0, 1)\cup (1,\beta_0^{-1}) $, it holds that
	 $
	 0<L\le \frac{f(\beta)}{g(\beta)} \le U < \infty$.
	 If $ \mathcal{H} $ forms 
	 an $ (r_1, r_2, p_1,p_2) $-sensitive family for $ 
	 g 
	 $-divergence on $ \mathcal{P} $, then it is also an  
	 $ (Lr_1, 
	 Ur_2, p_1,p_2)  
	 $-sensitive family for $ f $-divergence on $ \mathcal{P} $.
\end{proposition}

\cref{thm:main} provides a general strategy of constructing LSH families for $ 
f $-divergences. The performance of such LSH families depends on the tightness 
of 
the approximation. 
In \cref{sub:weighted_js,sub:triangular_discrimination}, as instances of the 
general strategy, we derive concrete 
results for the generalized Jensen-Shannon 
divergence and triangular discrimination, respectively.

\subsection{Generalized Jensen-Shannon Divergence}\label{sub:weighted_js}
First, \cref{lem:wjs-as-f-divergence} shows that the GJS divergence is indeed 
an 
instance of $ f $-divergence.
\begin{lemma}[{Proof in 
\cref{app:wjs-as-f-divergence}}]\label{lem:wjs-as-f-divergence}
	Define 
$	m_\lambda(t) = \lambda t \ln t - (\lambda t + 1 - 
	\lambda)\ln(\lambda t+1-\lambda)$.
	For any $ \lambda\in[0,1] $, $ m_\lambda(t) $ is convex on $ (0,\infty) $ 
	and $ m_\lambda(1)=0 $. Furthermore, $ m_\lambda $-divergence yields the 
	GJS divergence with parameter $ \lambda $.
\end{lemma}

We choose to approximate it via the squared Hellinger distance, which 
plays a central role in the construction of the hash family with desired 
properties.

The approximation guarantee is established in \cref{thm:approximation}. We show 
that the ratio of $ \wjs{\lambda}{P}{Q} $ to $ H^2(P,Q) $ is upper 
bounded by the function $ U(\lambda) $ and lower bounded by the function $ 
L(\lambda) $. Furthermore, \cref{thm:approximation} shows that $ U(\lambda)\le 
1 $, which implies that the squared Hellinger distance is an upper bound of the 
GJS divergence.
\begin{theorem}[{Proof in 
\cref{app:approximation}}]\label{thm:approximation}
	We assume that the sample space $ \Omega $ is finite. Let $ P $ and $ Q $ 
	be two different distributions on $ \Omega $.
	For every  $ 
	t>0 $ 
	and $ \lambda\in (0,1) $, we have
	\begin{equation*}
	L(\lambda)H^2(P, Q)\le 	\wjs{\lambda}{P}{Q}
	 \le  U(\lambda)H^2(P, Q)\le 
	H^2(P, 
	Q) ,
	\end{equation*}
	where $ L(\lambda) = 2 \min\{ \eta(\lambda), 
	\eta(1-\lambda) \} $, $ \eta(\lambda) = -\lambda \ln \lambda $ and $ 
	U(\lambda) = 
	\frac{2\lambda(1-\lambda)}{1-2\lambda}\ln \frac{1-\lambda}{\lambda} $.
\end{theorem}

We show \cref{thm:approximation} by showing a two-sided approximation result regarding $ m_\lambda $ and $ \hel $. This result might be of independent interest for other machine learning tasks, say, approximate information-theoretic clustering~\cite{chaudhuri2008finding}.
\begin{lemma}[Proof in \cref{app:bound}]\label{lem:bound}
	Define $ \kappa_\lambda(t) = \frac{m_\lambda(t)}{\hel(t)} $. For every 
	$ 
	t>0 $ 
	and $ \lambda\in (0,1) $, we have
		$
	\kappa_\lambda(t) = \kappa_{1-\lambda}(1/t)
	$ and $
	\kappa_\lambda(t) \in [L(\lambda), U(\lambda) ]
	$.
\end{lemma}

\begin{figure}[thb]
	\centering
	\includegraphics[width=0.35\columnwidth]{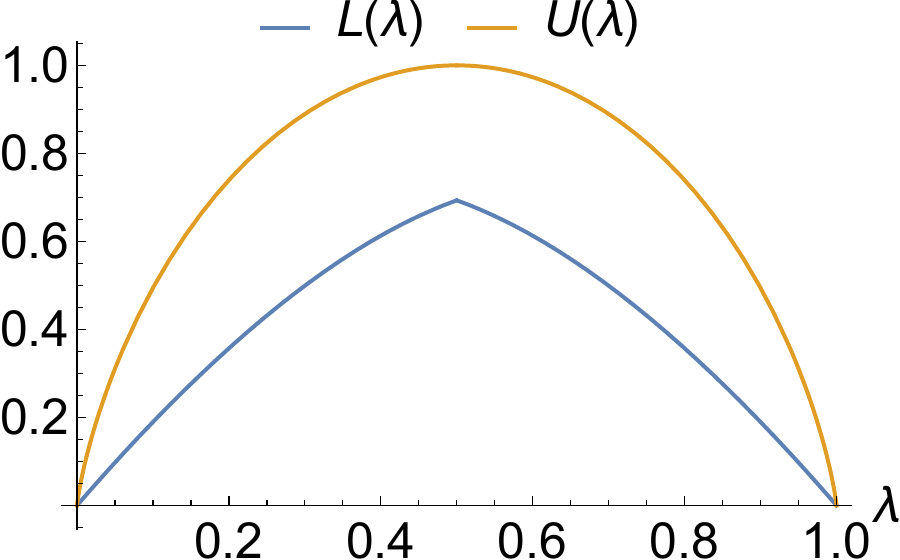}
	\caption{Upper and lower functions $ U(\lambda) $ and $ L(\lambda) 
		$.\label{fig:approximation_ratio}}
\end{figure}

We illustrate the upper and lower bound functions $ U(\lambda) $ and $ 
L(\lambda) $ in \cref{fig:approximation_ratio}. Recall that if $ \lambda=1/2 $, 
the generalized Jensen-Shannon divergence reduces to the usual Jensen-Shannon 
divergence. \cref{thm:approximation} yields the approximation 
guarantee
$
0.69< \ln 2 \le \frac{\js{P}{Q}}{H^2(P, Q)} \le 1$.

If the common sample space $ \Omega $ with which the 
two distributions $ P $ and $ Q $ are associated is finite, one can  
identify $ P $ and Q with the $ |\Omega| $-dimensional vectors $[P(i)]_{i\in 
\Omega}$ and $ [Q(i)]_{i\in \Omega} $, respectively.
In this case, 
$
H^2(P, Q) = \frac{1}{2} \| \sqrt{P}-\sqrt{Q} \|_2^2$,
which is exactly half of the squared $ L^2 $ distance between the two vectors $ 
\sqrt{P} \triangleq [\sqrt{P(i)}]_{i\in \Omega} $ and $ \sqrt{Q} \triangleq 
[\sqrt{Q(i)}]_{i\in \Omega} $. Therefore, the squared Hellinger distance can be 
endowed with the $ L^2 $-LSH family~\citep{datar2004locality} applied to the 
square root of the vector. In light of this, 
the locality-sensitive hash function that we propose for the generalized Jensen-Shannon  divergence is 
\begin{equation}\label{eq:hash_function}
 h_{\mathbf{a}, b}(P) = \left\lceil  \frac{\mathbf{a}\cdot \sqrt{P} +b }{r} 
 \right\rceil,
\end{equation}
where $ \mathbf{a}\sim \mathcal{N}(0, I) $ is a $ |\Omega| $-dimensional standard normal random vector, $ 
\cdot $ denotes the inner product, 
$ b $ is uniformly at random on $ [0, r] $, and $ r $ is a positive real 
number. 

\begin{theorem}[{Proof in \cref{sub:proof_wjs}}]\label{thm:lsh-family}
Let $ c= \| 
\sqrt{P}-\sqrt{Q} \|_2$ and $ f_2 $ be the probability density function of the 
absolute value of the standard normal distribution.	The hash functions $ 
\{h_{\mathbf{a},b}\}  $ defined in \eqref{eq:hash_function} form a 
$ (R,c^2\frac{U(\lambda)}{L(\lambda)}R,p_1,p_2) 
	$-sensitive 
	family for the 
	generalized Jensen-Shannon divergence with parameter $ \lambda $, where $ R>0 
	$, $ p_1 = p(1) $, $ p_2=p(c) $, and $
		p(u) = \int_0^r \frac{1}{u}f_2(t/u)(1-t/r)dt$.
\end{theorem}

\subsection{Triangular Discrimination\label{sub:triangular_discrimination}}
Recall that  triangular discrimination is the $ \delta $-divergence, where $ 
\delta(t) = \frac{(t-1)^2}{t+1} $.
As shown in the proof of \cref{thm:lsh_family_triangular_discrimination} 
(\cref{sub:proof_triangular_discrimination}), the function $ \delta $ can be 
approximated by the function $ \hel(t) $ that defines the squared Hellinger 
distance
$
1\le	\frac{\delta(t)}{\hel(t)} \le 2$.
The squared Hellinger distance can be sketched via $ L^2 $-LSH after taking the 
square root, as exemplified in \cref{sub:weighted_js}.
By \cref{thm:main}, the LSH family for the square Hellinger distance also forms 
an LSH family for the triangular discrimination.
\cref{thm:lsh_family_triangular_discrimination} shows that the LSH family 
defined in  \eqref{eq:hash_function} form a 
$ (R,2c^2R,p_1,p_2) 
$-sensitive 
family for 
triangular discrimination.

\begin{theorem}[{Proof in 
\cref{sub:proof_triangular_discrimination}}]\label{thm:lsh_family_triangular_discrimination}
	Let $ c= \| 
	\sqrt{P}-\sqrt{Q} \|_2$ and $ f_2 $ be the probability density function of 
	the 
	absolute value of the standard normal distribution.	The hash functions $ 
	\{h_{\mathbf{a},b}\}  $ defined in \eqref{eq:hash_function} form a 
	$ (R,2c^2R,p_1,p_2) 
	$-sensitive 
	family for 
	triangular discrimination, where $ R>0 
	$, $ p_1 = p(1) $, $ p_2=p(c) $, and $
	p(u) = \int_0^r \frac{1}{u}f_2(t/u)(1-t/r)dt$.
\end{theorem}
 
\section{\kr-LSH for Mutual Information Loss}
In this section, we first show that the mutual information loss is a \kr 
kernel. 
Then
we propose \emph{\kr-LSH}, an asymmetric LSH 
method~\cite{shrivastava2014asymmetric} for mutual information 
loss. We would like 
to remark that this method can be easily extended to other \kr kernels, 
provided that the associated positive definite kernels allow an explicit 
feature map. 

\subsection{Mutual Information Loss is a \kr Kernel}

Recall that in \cref{sub:mil-gjs} we assume a joint distribution $ p(X,C) $ 
whose support is $ \cX\times \cC $. Let $ x,y\in \cX $ be represented by 
 $\bfx=[p(c,x):c\in \cC]\in [0,1]^{|\cC|}$ and $\bfy=[p(c,y):c\in \cC]\in 
 [0,1]^{|\cC|}$, respectively.
We consider the mutual information loss of merging $ x  $ and $ y $, which is given by
$ I(X;C)-I(\pi_{x,y}(X);C)$. 

\begin{theorem}[Proof in \cref{app:krein_kernel}]\label{thm:krein_kernel}
	The mutual information loss $ \mil(\bfx,\bfy) $ is a \kr kernel on $ 
	[0,1]^{|\cC|} $. In other words, there exist two positive definite kernels 
	$ K_1 $ and $ K_2 $ on $ [0,1]^{|\cC|} $ such that $ 
	\mil(\bfx,\bfy)=K_1(\bfx,\bfy)-K_2(\bfx,\bfy) $. To be explicit, we set $ 
	K_1(\bfx,\bfy) = k(\sum_{c\in \cC}p(c,x), \sum_{c\in \cC}p(c,y)) $ and $ 
	K_2(\bfx,\bfy)=\sum_{c\in \cC} k(p(c,x),p(c,y)) $, where
	$
	k(a,b) = a\ln\frac{a}{a+b}+b\ln\frac{b}{a+b}
	 $.
\end{theorem}

To prove \cref{thm:krein_kernel} and construct explicit feature maps for $ K_1 $ and $ K_2 $, we need the following lemma.

\begin{lemma}[Proof in \cref{app:pd_k}]\label{lem:pd_k}
	The kernel $ k $ is a positive definite kernel on $ [0,1] $. Moreover, it is endowed with the 
	following explicit feature map $ x\mapsto \Phi_w(x) $ such that 
	$ k(x,y) =  \int_\bR \Phi_w(x)^* \Phi_w(y) dw $, 
	 where $
	\Phi_w(x) \triangleq e^{-iw\ln(x)}\sqrt{x\frac{2\sech(\pi w)}{1+4w^2}}$ and $ \Phi_w(x)^* $ denotes the complex conjugate of $ \Phi_w(x) $.
\end{lemma}

The map $ \Phi(x):w\mapsto \Phi_w(x) $ is called the \emph{feature map} of $ x 
$. The integral $ \int_\bR \Phi_w(x)^* \Phi_w(y) dw $ is also denoted by a 
Hermitian inner product $ \langle \Phi(x),\Phi(y)\rangle $.

\subsection{\kr-LSH for Mutual Information Loss}

Now we are ready to present an asymmetric LSH 
scheme~\cite{shrivastava2014asymmetric} for mutual information 
loss. 
This method can be easily extended to other \kr kernels, provided that the 
associated positive definite kernels admit an explicit feature map. In fact, we reduce the problem of designing the LSH for a \kr kernel to the problem of designing the LSH for maximum inner product search (MIPS)~\citep{shrivastava2014asymmetric,neyshabur2015symmetric,yan2018norm}.
We call this general reduction \emph{\kr-LSH}.

\subsubsection{Reduction to Maximum Inner Product Search}

Our reduction is based on the following observation. Suppose that $ K $ is a 
\kr kernel on $ \cX $ such that $ K=K_1-K_2 $ where $ K_1 $ and $ K_2 $ are 
positive definite kernels on $ \cX $. Assume that $ K_1 $ and $ K_2$ admit 
feature maps $ \Phi_1 $ and $\Phi_2 $ such that $ K_1(x,y)=\langle 
\Psi_1(x),\Psi_1(y)\rangle $ and $ K_2(x,y)=\langle \Psi_2(x),\Psi_2(y)\rangle 
$. Then the \kr kernel $ K $ can also represented as an inner product 
\begin{equation}\label{eq:krein_inner_product}
K(x,y)=\langle \Phi_1(x)\oplus \Phi_2(x), \Phi_1(y)\oplus -\Phi_2(y)\rangle\enspace,
 \end{equation}
 where $ \oplus $ denotes the direct sum. If we define a pair of transforms $ 
 T_1(x) \triangleq \Phi_1(x)\oplus \Phi_2(x)
 $ and $ T_2(x)\triangleq \Phi_1(x)\oplus -\Phi_2(x)
  $,
  then we have $
  K(x,y)=\langle T_1(x),T_2(y)\rangle
   $.
   We call this pair of transforms \emph{left and right \kr transforms}.
   
\begin{algorithm}[tbh]
	\caption{ \kr-LSH}\label{alg:krein-lsh}
	\begin{algorithmic}[1]
		\Require Discretization parameters $ J\in \mathbb{N} $ and $ \Delta>0 $.
		\Ensure The left and right \kr transform $ \eta_1 $ and $ \eta_2 $.
		\State $ w_j \gets (j-1/2)\Delta $ for $ j=1,\dots,J $
		\State Construct the atomic transform \[
		\tau(x,w,j)\triangleq 
		\left[\cos(w\ln(x))\sqrt{2x\int_{(j-1)\Delta}^{j\Delta} 
			\rho(w')dw'},\sin(w\ln(x))\sqrt{ 2x\int_{(j-1)\Delta}^{j\Delta} 
			\rho(w')dw' }\right]
		\enspace.
		\]
		\State Construct the left and right basic transform \begin{align*}
		\eta_1(\bfx)\triangleq{} & \bigoplus_{j=1}^J \tau(p(x),w_j,j)\oplus 
		\bigoplus_{j=1}^{J} \bigoplus_{c\in \cC} \tau(p(c,x),w_j,j)\,,\\
		\eta_2(\bfx)\triangleq{} & \bigoplus_{j=1}^J \tau(p(x),w_j,j)\oplus 
		\bigoplus_{j=1}^{J} \bigoplus_{c\in \cC} -\tau(p(c,x),w_j,j)\,.
		\end{align*}
		\State Construct the left and right \kr transform
		\begin{equation*}
		T_1(\bfx,M) \triangleq{} 
		[\eta_1,\sqrt{M-\|\eta_1(\bfx)\|^2_2},0],\quad T_2(\bfy,M) \triangleq{} 
		[\eta_2,0,\sqrt{M-\|\eta_2(\bfx)\|^2_2}]\,.
		\end{equation*}
		where $ M $ is a constant such that $ M\ge \|\eta_1(\bfx)\|_2^2 $ (note 
		that $ 
		\|\eta_1(\bfx)\|_2 = \|\eta_2(\bfx)\|_2 $).
		\State Sample $\bfa\sim \cN(0,I)$ and construct the hash 
		function 
		$ h(\bfx; M)\triangleq \sign(\bfa^\top T(\bfx, M))$,
		where $T$ is either the left or right transform.\label{ln:simple-lsh}
	\end{algorithmic}
\end{algorithm}
   
   We exemplify this technique by applying it to the MIL divergence.
For ease of exposition, we define $ \rho(w)\triangleq \frac{2\sech(\pi 
w)}{1+4w^2} $. The proposed approach \kr-LSH is presented in 
\cref{alg:krein-lsh}. 
To make the intuition of \eqref{eq:krein_inner_product} applicable in a practical implementation, we have
 to truncate and discretize the 
integral $ k(x,y) = \int_R \Phi_w(x)^* \Phi_w(y) 
dw  $. First we analyze the truncation. The analysis is similar to Lemma~10 of \citep{abdullah2016sketching}.
\begin{lemma}[Truncation error bound, proof in \cref{app:truncation}]\label{lem:truncation}
	If $ t>0 $ and $ x,y\in [0,1] $, the truncation error can be bounded as 
	follows $
	\left| k(x,y) - \int_{-t}^{t} \Phi_w(x)^* \Phi_w(y)  dw \right|  \le 
	4e^{-t}$.
\end{lemma}
To discretize the finite integral $ \int_{-t}^{t} \Phi_w(x)^* \Phi_w(y)  dw $, we divide the inteval into $ 2J $ sub-intervals of length $ \Delta $. The following lemma bounds the discretization error.
\begin{lemma}[Discretization error bound, proof in \cref{app:discretization}]\label{lem:discretization}
	If $ J $ is a positive integer, $ \Delta>0 $, and $ w_j=(j-1/2)\Delta $, 
	the discretization error is bounded as follows \[
	\left| \int_{-\Delta J}^{\Delta J} \Phi_w(x)^* \Phi_w(y)  dw - 
	\left\langle \bigoplus_{j=1}^J \tau(x,w_j,j), \bigoplus_{j=1}^J \tau(y,w_j,j)\right\rangle\right|
	 \le 2\Delta
	 \,,\]
	 where $ \tau(x,w,j) = 
	 \left[\cos(w\ln(x))\sqrt{2x\int_{(j-1)\Delta}^{j\Delta} 
	 	\rho(w')dw'},\sin(w\ln(x))\sqrt{ 2x\int_{(j-1)\Delta}^{j\Delta} 
	 	\rho(w')dw' }\right]\in \bR^2 $.
\end{lemma}

By \cref{lem:truncation,lem:discretization}, to guarantee that the total approximation error (including both truncation and discretization errors) is at most $ \epsilon $, it suffices to set $ \Delta = \frac{\epsilon}{4(1+|\cC|)} $ and $ J\ge \frac{4(1+|\cC|)}{\epsilon}\ln\frac{8(1+|\cC|)}{\epsilon} $.

\subsubsection{LSH for Maximum Inner Product Search}

The second stage of our proposed method is to apply LSH to the MIPS problem. 
As an example,  
in \cref{ln:simple-lsh}, we use the \splsh introduced by 
\citep{neyshabur2015symmetric}. Let us have a quick review of \splsh. Assume 
that $ \cM\subseteq \bR^d $ is a finite set of vectors and that for all $ 
\bfx\in \cM $, there is a universal bound on the squared 2-norm, \ie, $ \|\bfx 
\|^2_2\le M $.  \citet{neyshabur2015symmetric} assume that $ M=1 $ without loss 
of generality. We allow $ M $ to be any positive real number. 
For two vectors $ \bfx,\bfy\in \cM $, \splsh performs the following transform $
L_1(\bfx) \triangleq{} 
[\bfx,\sqrt{M-\|\bfx\|^2_2},0], L_2(\bfy) \triangleq{} 
[\bfy,0,\sqrt{M-\|\bfy\|^2_2}]$.
 Note that the norm of $ L_1 $ and $ L_2 $ is $ M $ and that therefore their cosine similarity equals their inner product. 
In fact, \splsh is a reduction from MIPS to LSH for the cosine similarity. Then a random-projection-based LSH for the cosine similarity~\citep{charikar2002similarity,wang2014hashing} \[ 
h(\bfx)\triangleq \sign(\bfx^\top L_i(\bfx)),\quad \bfa\sim \cN(0, I), i=1,2
 \]
 can be used for MIPS and thereby LSH for the MIL divergence via our reduction.

\paragraph{Discussion}
We have some important remarks for practical implementation of \kr-LSH. 
Although \citep{neyshabur2015symmetric} provides a theoretical guarantee for 
LSH for MIPS, as noted in~\citep{yan2018norm}, the additional term $ 
\sqrt{M-\|\bfx\|_2^2} $ may dominate in the $ 2 $-norm and significantly 
degrade the performance of LSH. To circumvent this issue, we recommend a method 
that partitions the dataset according to the $ 2 $-norm, \eg, the norm-ranging 
method~\citep{yan2018norm}.

\section{Experiment Results}
\begin{figure*}[bht]
	\centering
	\begin{subfigure}[t]{0.3\linewidth}
		\includegraphics[width=\linewidth]{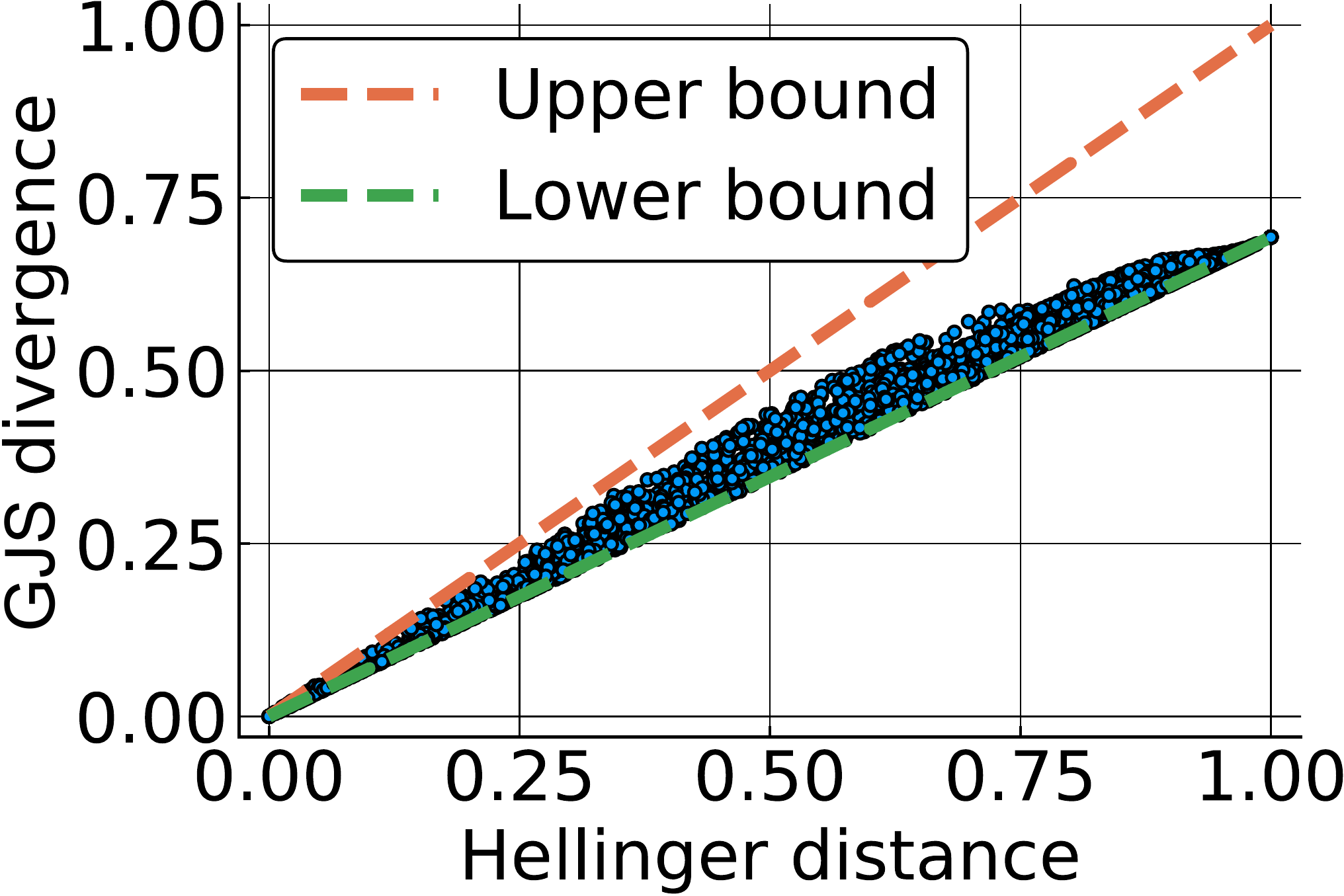}
		\caption{$ \lambda=1/2 $\label{fig:approximation2}}
	\end{subfigure}
\begin{subfigure}[t]{0.3\linewidth}
	\includegraphics[width=\linewidth]{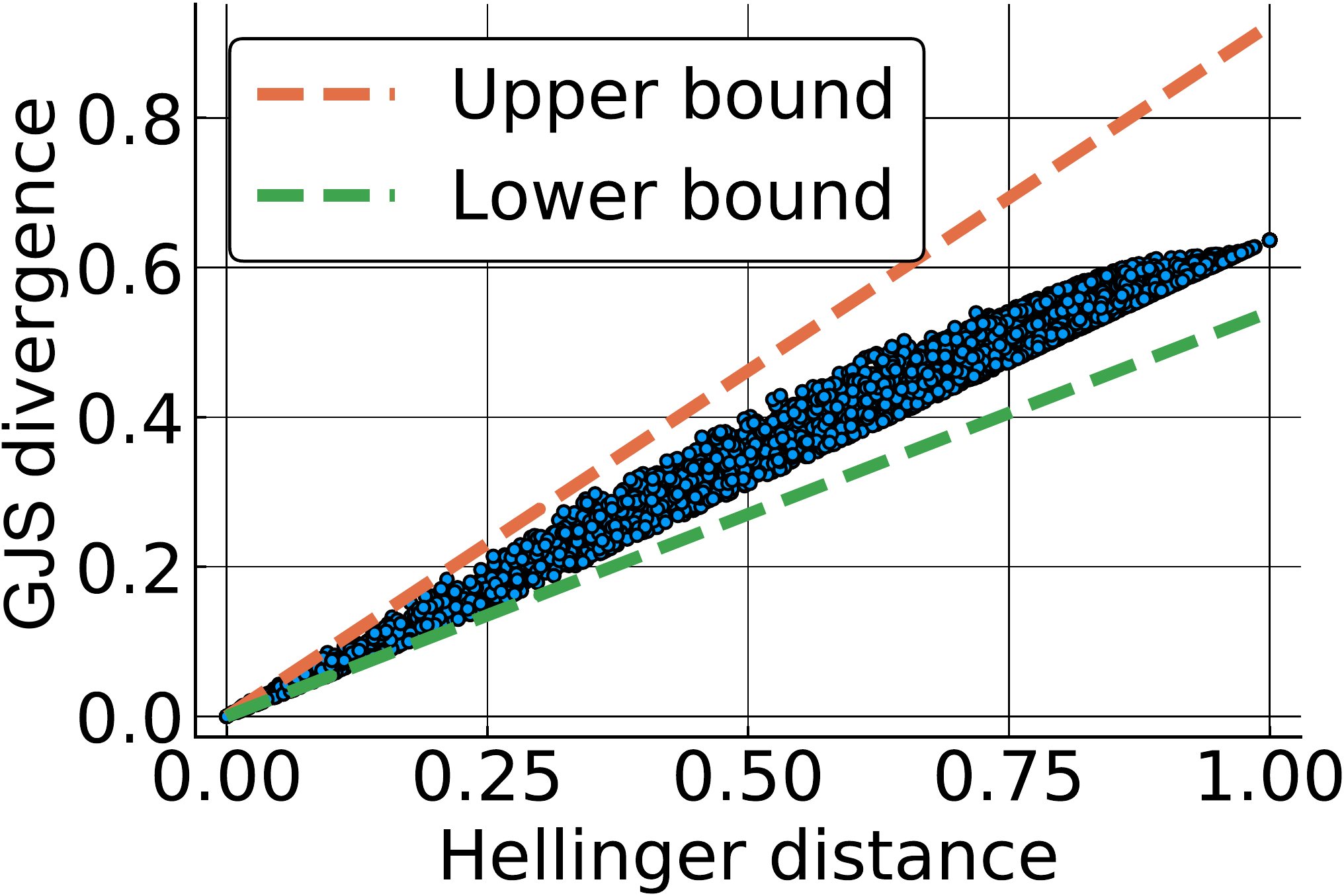}
	\caption{$ \lambda=1/3 $\label{fig:approximation3}}
\end{subfigure}
\begin{subfigure}[t]{0.3\linewidth}
	\includegraphics[width=\linewidth]{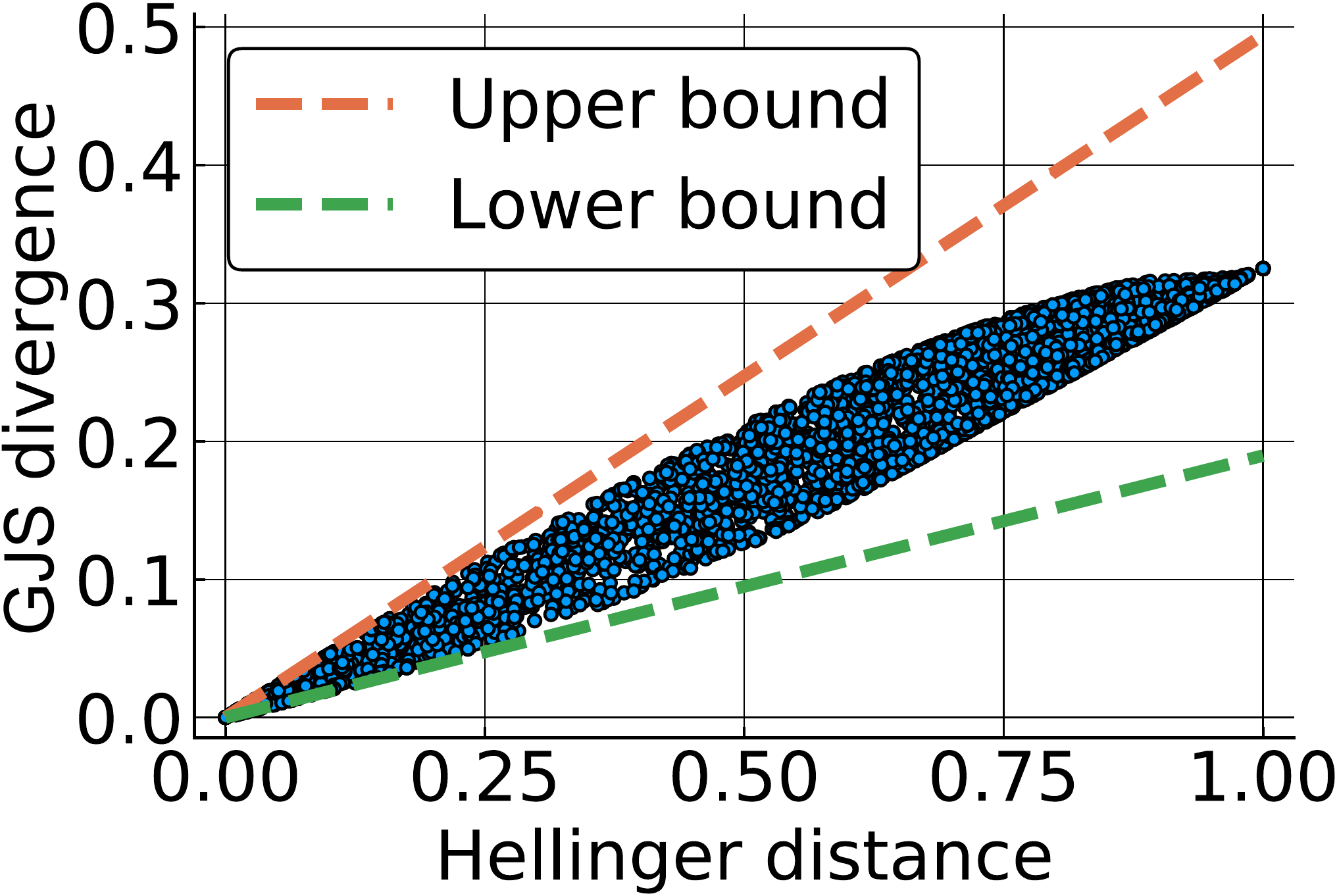}
	\caption{$ \lambda=1/10 $\label{fig:approximation10}}
\end{subfigure}
\caption{The empirical performance of Hellinger 
approximation\label{fig:approximation}}
\end{figure*}
\begin{figure*}[hbt]
	\centering
	\begin{subfigure}[t]{0.3\linewidth}
		\includegraphics[width=\linewidth]{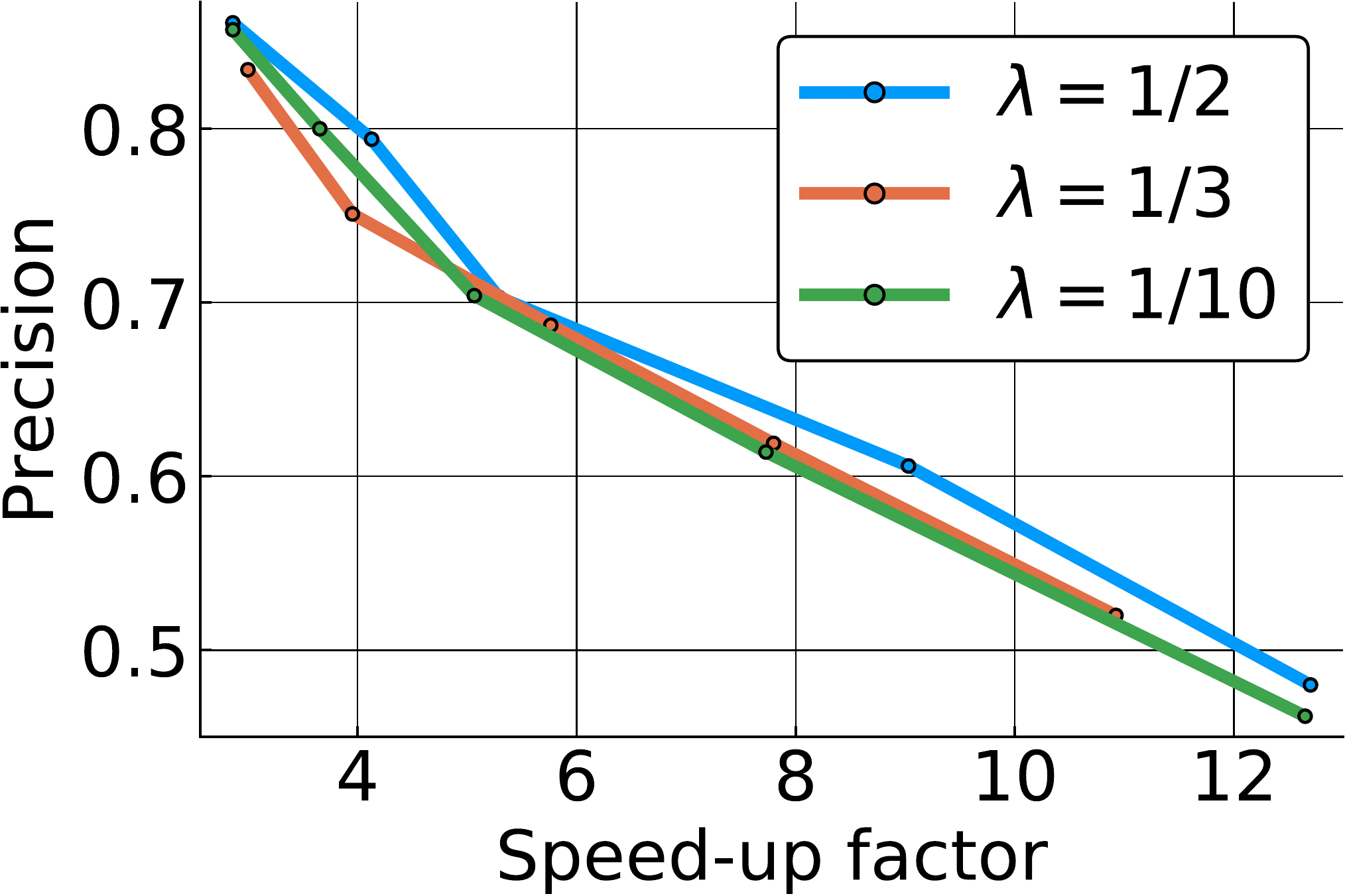}
		\caption{Fashion MNIST\label{fig:lambda_fashion}}
	\end{subfigure}
\begin{subfigure}[t]{0.3\linewidth}
	\includegraphics[width=\linewidth]{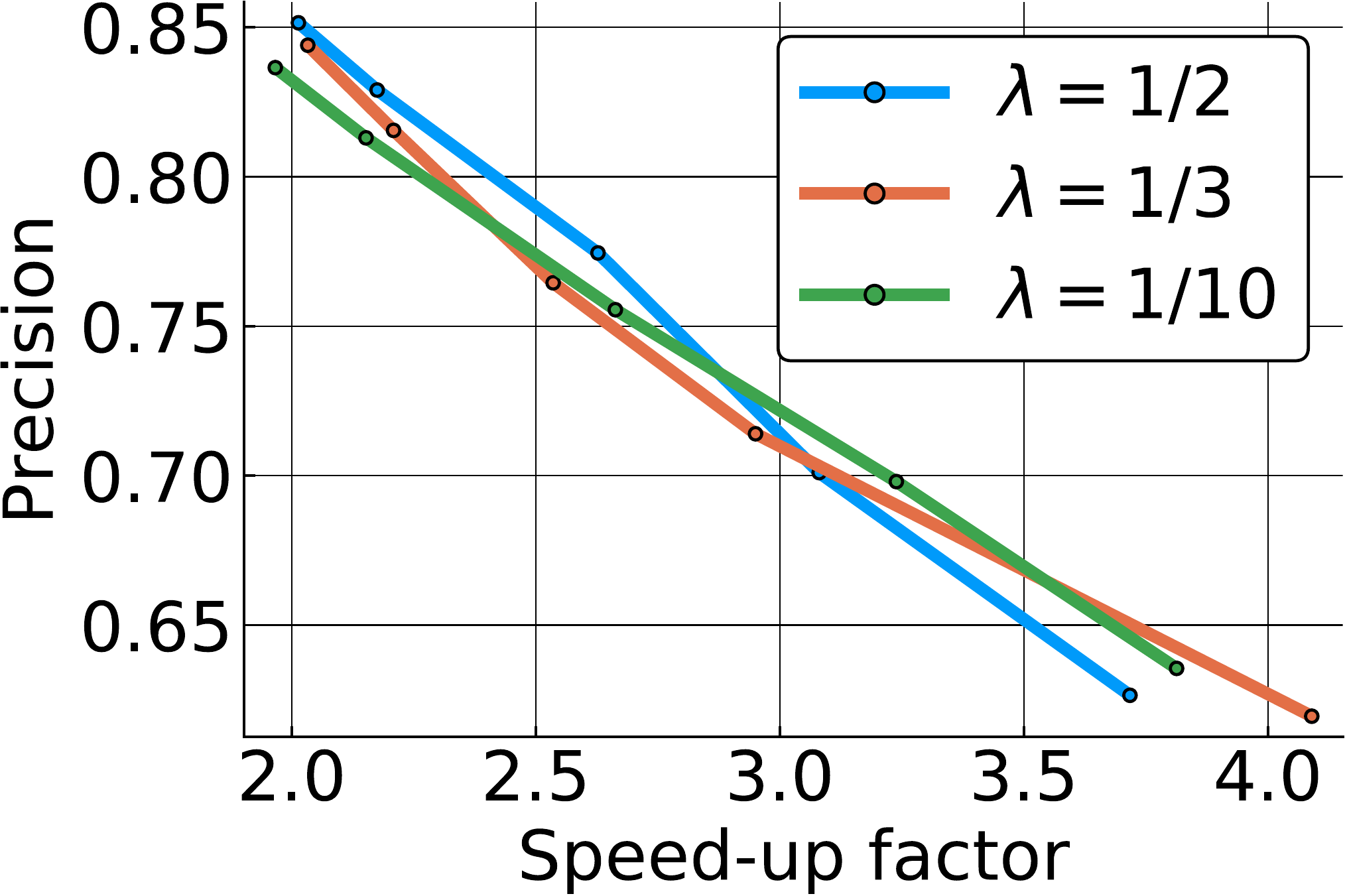}
	\caption{MNIST\label{fig:lambda}}
\end{subfigure}
\begin{subfigure}[t]{0.3\linewidth}
	\includegraphics[width=\linewidth]{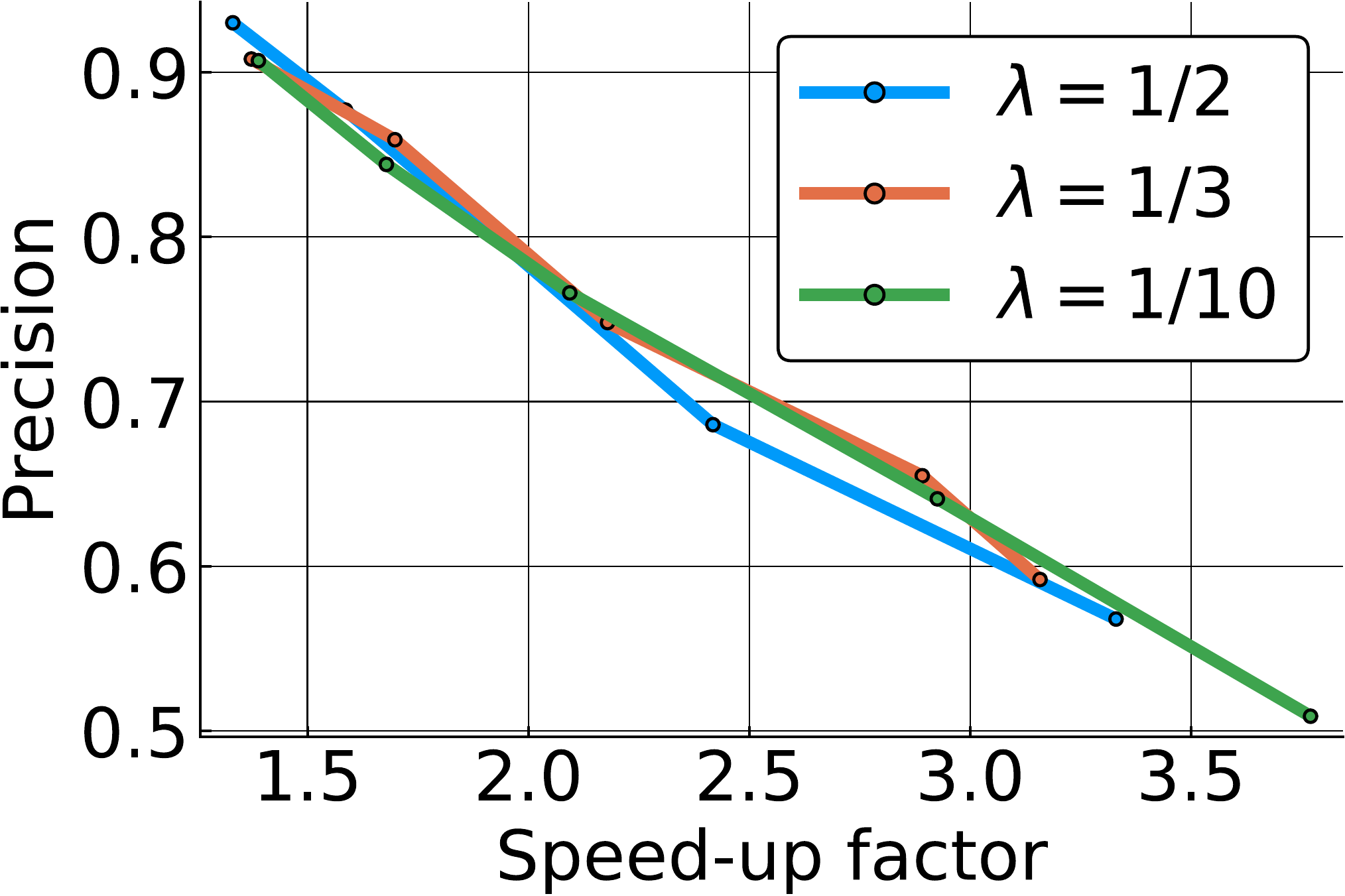}
	\caption{CIFAR-10\label{fig:lambda_cifar}}
\end{subfigure}
\caption{Precision vs.\ speed-up factor for different $ \lambda $'s.
}
\end{figure*}
\begin{figure*}[hbt]
	\centering
	\begin{subfigure}[t]{0.3\linewidth}
		\includegraphics[width=\linewidth]{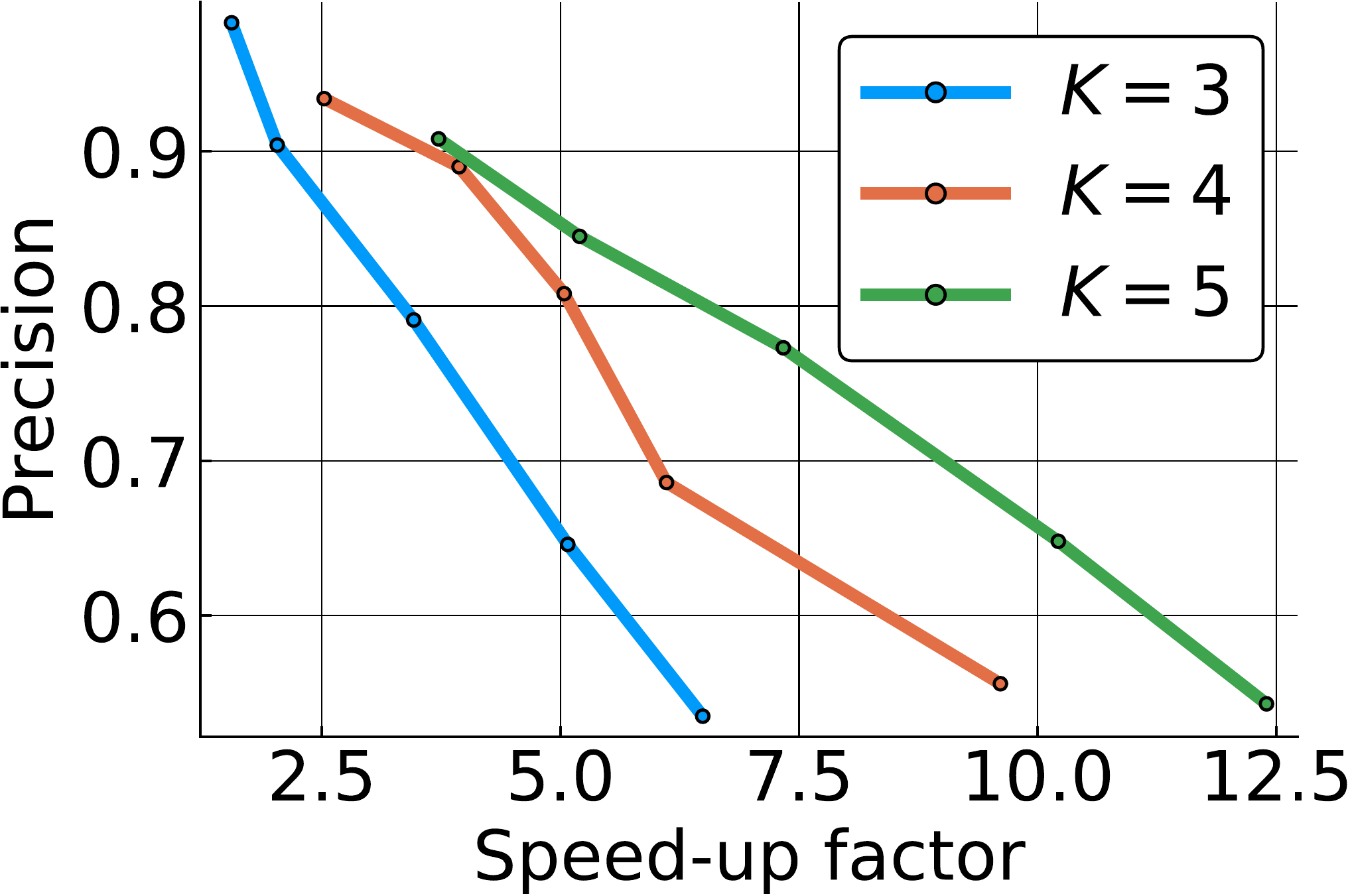}
		\caption{Fashion MNIST\label{fig:size_sketches_fashion}}
	\end{subfigure}
	\begin{subfigure}[t]{0.3\linewidth}
		\includegraphics[width=\linewidth]{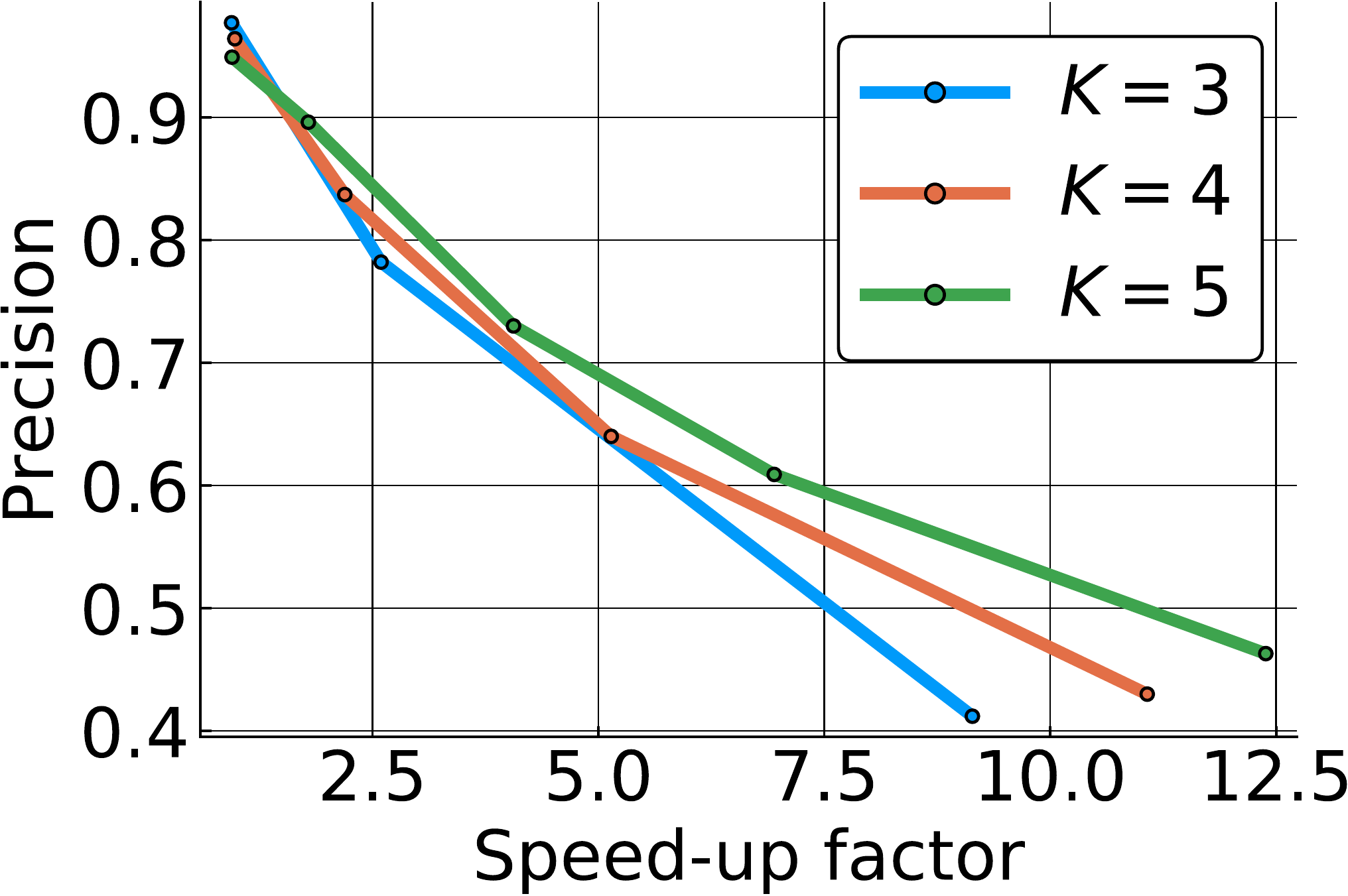}
		\caption{MNIST\label{fig:size_sketches}}
	\end{subfigure}
	\begin{subfigure}[t]{0.3\linewidth}
		\includegraphics[width=\linewidth]{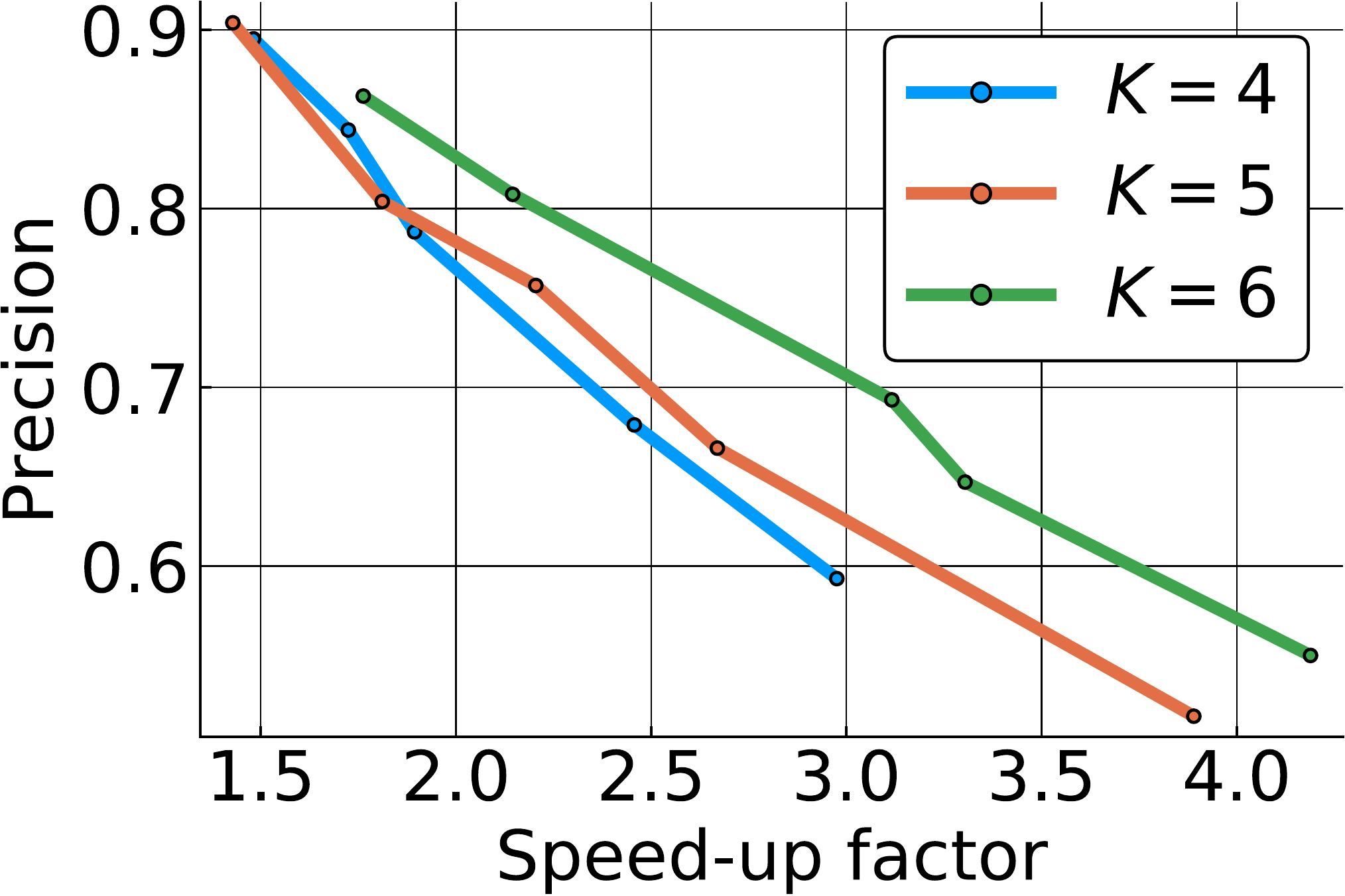}
		\caption{CIFAR-10\label{fig:size_sketches_cifar}}
	\end{subfigure}
	\caption{Precision vs.\ speed-up factor for different  sketch sizes.
		\label{fig:precision-vs-sketch-size}
	}
\end{figure*}

\textbf{Approximation Guarantee.}
In the first part, we verify the theoretical bounds derived in 
\cref{thm:approximation} on real data. We used the 
latent Dirichlet allocation to extract the 
topic distributions of 
 Reuters-21578, 
Distribution 1.0. The number of topics is set to $ 10 $. We sampled $ 100 $ 
documents uniformly at random and computed the GJS divergence and Hellinger 
distance between each pair of topic distributions. Each dot in 
\cref{fig:approximation} represents the topic distribution of a document. The 
horizontal axis denotes the Hellinger distance while the vertical axis denotes 
the GJS divergence. We chose different parameter values ($ \lambda=1/2,1/3,1/10 
$) for the GJS divergence. From the three subfigures, we observe that both 
the upper and lower bounds are tight for the data.

\textbf{Nearest Neighbor Search.}
In the second part, we apply the proposed LSH scheme for the GJS divergence to 
the nearest neighbor search problem in Fashion 
MNIST~\citep{xiao2017/online}, MNIST~\citep{lecun1998gradient}, and 
CIFAR-10~\citep{krizhevsky2009learning}. 
 Each image in the 
datasets is flattened into a vector and $ L^1 $-normalized, 
thereby summing to $ 1 $. As described in \cref{sub:prelim_lsh}, a  
concatenation of hash functions is used. 
 We 
denote  the number of concatenated hash 
functions by $ K $ and the number of compound hash functions by $ L $. 
In the 
first set of experiments, we set $ 
K=3 $ and vary $ L $ from $ 20 $ to $ 40 $. We measure the execution time of 
LSH-based $ k $-nearest neighbor search and the exact (brute-force) 
algorithm, where $ k $ is set to $ 20 $. Both 
algorithms were run on a 2.2 GHz Intel Core i7 processor. The speed-up factor 
is the ratio of the execution time of the exact algorithm to that of the 
LSH-based method. The quality of the result returned by the LSH-based method is 
quantified by its precision, which is the fraction of correct nearest neighbors 
among the retrieved items. We would like to remark that the precision and 
recall are equal in our case since both algorithms return $ k $ items. We also 
vary the parameter of the GJS divergence and 
choose $ \lambda $ from $ \{ 1/2,1/3,1/10 \} $. The result is illustrated in 
\cref{fig:lambda_fashion,fig:lambda,fig:lambda_cifar}. We observe a trade-off 
between the quality of the output 
(precision) and computational efficiency (speed-up factor). The performance 
appears to be robust to the parameter of the GJS divergence. 
In the second set of experiments, we fix the parameter of the GJS divergence to 
$ 1/2 $; \ie, the JS divergence is used. The number of concatenated hash 
functions $ K $ ranges from $ 3 $ to $ 5 $ or $ 4 $ to $ 6 $. The result is 
presented in
\cref{fig:size_sketches_fashion,fig:size_sketches,fig:size_sketches_cifar}.
 In 
addition to the aforementioned quality-efficiency 
trade-off, we observe that a larger $ K $ results in a more efficient algorithm 
given the same target precision.

\section{Conclusion}
In this paper, we propose a general strategy of designing an LSH family for 
$f$-divergences.
We exemplify this strategy by developing LSH schemes for the generalized 
Jensen-Shannon divergence and triangular discrimination in this framework. They 
are endowed with an LSH family via the Hellinger approximation. In particular, 
we show a two-sided approximation for the generalized Jensen-Shannon divergence 
by the Hellinger distance. This may be of independent interest.
Next, we propose a general approach to designing an LSH scheme for \kr kernels via a reduction to the problem of maximum inner product search. In contrast to our strategy for $ f $-divergences, this approach involves no approximation and is theoretically lossless. We exemplify this approach by applying to mutual information loss.

\subsubsection*{Acknowledgments}
LC was supported by the Google PhD Fellowship.

\clearpage
\printbibliography

\clearpage
\begin{appendices}
	\crefalias{section}{appsec}
	\crefalias{subsection}{appsec}
\section{Proof of \cref{lem:bound}}\label{app:bound}
	The first equation 
	$
	\kappa_\lambda(t) = \kappa_{1-\lambda}(1/t)
	$
	 can be verified directly by plugging in $ 
	1-\lambda $ and $ 1/t $. In the sequel, we show the second equation 
$
	\kappa_\lambda(t) \in [L(\lambda), U(\lambda) ]
	$,
	which needs a detailed and careful analysis and discussion.
	The derivative of $ \kappa_\lambda $, denoted by $ \kappa'_\lambda(t)  $, 
	is \begin{equation*}
\frac{2 \left(\lambda  \left(\sqrt{t}-1\right)+1\right) 
\ln (\lambda  (t-1)+1)-2 \lambda  \sqrt{t} \ln (t)}{\left(\sqrt{t}-1\right)^3 
\sqrt{t}}.
	\end{equation*}
	We define $ f_1(t) = 2 \left(\lambda  \left(\sqrt{t}-1\right)+1\right) 
	\ln (\lambda  (t-1)+1)-2 \lambda  \sqrt{t} \ln (t) $.
	Its derivative $ f_1'(t) $ is 
	\begin{equation*}
	 -\frac{\lambda}{\sqrt{t} (\lambda  (t-1)+1)}  \left(2 (\lambda -1) 
	\left(\sqrt{t}-1\right)
	+(\lambda (t-1)+1) (\log (t)- \log (\lambda (t-1)+1))
	\right).
	\end{equation*}
	Define $ f_2(t) = 2 (\lambda -1) 
	\left(\sqrt{t}-1\right)
	+(\lambda (t-1)+1) (\log (t)- \log (\lambda (t-1)+1)) $. Its derivative $ 
	f'_2(t) $ is
	\begin{equation*}
		\frac{(\lambda -1) \left(\sqrt{t}-1\right)}{t}+\lambda  (\log 
		(t)-  \log (\lambda  (t-1)+1)).
	\end{equation*}
	Its second derivative $ f''_2(t) $ is 
	\begin{equation*}
		\frac{(1-\lambda ) \left(2 (\lambda -1)+\sqrt{t} (\lambda  
		(t-1)+1)\right)}{2 t^2 (\lambda  (t-1)+1)}.
	\end{equation*}
	First, we assume $ \lambda\in (0,1/2) $. In this case, we have $ 
	\frac{1-\lambda}{\lambda}> 1 $ and $ \lambda  (t-1)+1 > 0 $. Notice that $ 
	f_3(t) = 2 (\lambda 
	-1)+\sqrt{t} (\lambda  
(t-1)+1) $ is a strictly increasing function in $ t $. Therefore, if $ t> 
\left(\frac{1-\lambda }{\lambda }\right)^2 $, we obtain 
\begin{equation*}
	f_3(t) > f_3\left(\left(\frac{1-\lambda }{\lambda }\right)^2\right) = 
	\frac{(\lambda -1) (\lambda +1) (2 \lambda -1)}{\lambda ^2} > 0.
\end{equation*}
Therefore $ f''_2(t) > 0 $ if $ t> \left(\frac{1-\lambda }{\lambda }\right)^2 
$. Thus we deduce that $ f'_2(t) $ is increasing in $ t $ if $ t> 
\left(\frac{1-\lambda }{\lambda }\right)^2 $, which yields 
\begin{equation*}
	f'_2(t) > f'_2\left(  \left(\frac{1-\lambda }{\lambda }\right)^2  \right) = 
	\frac{\lambda  \left(2 \lambda +(1-\lambda)  \log 
	\left(\frac{1-\lambda}{\lambda 
	}\right)-1\right)}{1-\lambda 
	}.
\end{equation*}
Define $ g(\lambda) = 2 \lambda +(1-\lambda)  \log 
\left(\frac{1-\lambda}{\lambda 
}\right)-1 $. Its derivative $ g'(\lambda) = -\frac{1}{\lambda }-\log 
\left(\frac{1}{\lambda }-1\right)+2$ is negative if $ \lambda < 1/2 $ and 
positive if $ \lambda > 1/2 $. Therefore $ g(\lambda) \ge g(1/2) = 0 $. Thus we 
obtain that if $ t > \left(\frac{1-\lambda }{\lambda }\right)^2 $, $ f'_2(t)>0 
$, which implies that $ f_2(t) $ is increasing in $ t $ if $ t > 
\left(\frac{1-\lambda }{\lambda }\right)^2 $. Thus we have 
\begin{equation*}
	 f_2(t)> f_2\left(  
	\left(\frac{1-\lambda }{\lambda }\right)^2 \right) = 
	\frac{(1-\lambda ) \left(4 \lambda +\log \left(\frac{1}{\lambda 
	}-1\right)-2\right)}{\lambda }.
\end{equation*}
Define $ g_1(\lambda) = 4 \lambda +\log \left(\frac{1}{\lambda 
}-1\right)-2 $. Its derivative $ g'_1(t) = \frac{1}{(\lambda -1) \lambda }+4 $ 
is non-positive, which implies that $ g_1 $ is decreasing in $ \lambda $. 
Therefore, if $ t> \left(\frac{1-\lambda }{\lambda }\right)^2 $, we have 
\begin{equation*}
	f_2(t) > \frac{1-\lambda}{\lambda}g(1/2)=0.
\end{equation*}
Since $ \lambda(t-1)+1>0 $, we obtain that $ f'_1(t) < 0 $ and therefore $ 
f_1(t) $ is decreasing if $ t> \left(\frac{1-\lambda }{\lambda }\right)^2 $. We 
have 
\begin{equation*}
	f_1(t)< f_1\left( \left(\frac{1-\lambda }{\lambda }\right)^2 \right)=0.
\end{equation*}
If $ t> \left(\frac{1-\lambda }{\lambda }\right)^2 $,
since $ \left(\sqrt{t}-1\right)^3 
\sqrt{t} > 0 $, we deduce that $ \kappa'_\lambda(t) <0 $.

If $ t < 1 $, since $ f_3(t) $ is strictly increasing in $ t $, we have $ 
f_3(t) < f_3(1) = 2\lambda -1 < 0 $, which implies that $ f''_2(t) <0 $. 
Therefore, we obtain that $ f'_2(t) $ is strictly decreasing on $ (0,1) $. Thus 
we have $ f'_2(t) > f'_2(1) = 0 $, which implies that $ f_2(t) $ is strictly 
increasing on $ (0,1) $. We immediately have $ f_2(t) < f_2(1) = 0 $ for $ 
\forall t\in (0,1) $, which yields that $ f'_1(t) > 0 $ and therefore $ f_1(t) 
$ is strictly increasing on $ (0,1) $. For $ \forall t\in (0,1) $, it holds 
that $ f_1(t) < f_1(1) = 0 $. Since $ \left(\sqrt{t}-1\right)^3 
\sqrt{t} < 0 $, we deduce that $ \kappa'_\lambda(t) > 0 $ for $ t\in (0,1) $.

The interval that remains unexplored is $ I = \left(1, \left( 
\frac{1-\lambda}{\lambda} 
\right)^2 \right) $. Since $ f_3(1) = 2 \lambda -1 <0 $ and $ f_3\left( \left( 
\frac{1-\lambda}{\lambda} 
\right)^2 \right) = \frac{(\lambda -1) (\lambda +1) (2 \lambda -1)}{\lambda 
^2}>0 
$, we know that $ f_3(t) $ has a real root on this interval.
Notice that $ f_3(t) $ can be viewed as a cubic function 
in $ \sqrt{t} $. 
Define $ f_4(x) = 2 \lambda +\lambda  x^3+(1-\lambda ) x-2 $ and we have $ 
f_3(t) = f_4(\sqrt{t}) $. The cubic function $ f_4 $ is strictly monotone if $ 
\lambda\in (0,1) $. Therefore, the real root of $ f_3 $ on $ I $ is unique and 
we denote it by $ \rho(\lambda) $.

Now we divide the interval $I =  \left(1, \left( 
\frac{1-\lambda}{\lambda} 
\right)^2 \right) $ into two subintervals $ I_1 = (1,\rho(\lambda)) $ and $ 
I_2 = \left(\rho(\lambda), \left( 
\frac{1-\lambda}{\lambda} 
\right)^2 \right) $. Since $ f_3(t) <0 $ on $ I_1 $ and $ f_3(t)>0 $ on $ I_2 
$, we have $ f''_2(t) <0 $ on $ I_1 $ and $ f''_2(t) >0 $ on $ I_2 $. 
Therefore, we deduce that $ f'_2(t) $ strictly decreases on $ I_1 $ and 
strictly increases on $ I_2 $. Note that $ f'_2(1) = 0 $ and 
\begin{equation*}
	f'_2\left( \left(\frac{1-\lambda}{\lambda}\right)^2 \right) = 
	\frac{\lambda  \left(2 \lambda +(1-\lambda)  \log 
	\left(\frac{1-\lambda}{\lambda 
	}\right)-1\right)}{1-\lambda 
	}> 0.
\end{equation*}
To see this, we define $ g_2(\lambda) = 2 \lambda +(1-\lambda)  \log 
\left(\frac{1-\lambda}{\lambda 
}\right)-1 $. Its second derivative is $ g''_2(\lambda) = \frac{1}{\lambda 
^2-\lambda ^3} >0 $, which implies that $ g_2(\lambda) $ is strictly convex and 
$ g'_2(\lambda) $ has a unique root. Observe that $ \lambda=1/2 $ is a root of 
$ g'_2(\lambda) $. We deduce that $ g_2(\lambda) > g_2(1/2) = 0 $ for $ 
\lambda\in (0,1/2) $, which immediately yields that $ f'_2\left( 
\left(\frac{1-\lambda}{\lambda}\right)^2 \right) > 0 $.
Thus the function $ f'_2(t) $ has a unique root (denoted by $ \rho_1(\lambda) 
$) on 
$ I 
$. Therefore, the function $ f_2(t) $ strictly decreases on $ I_3 = 
(1,\rho_1(\lambda)) $ and strictly increases on $ I_4 = \left(\rho_1(\lambda), 
\left( 
\frac{1-\lambda}{\lambda} 
\right)^2 \right) $. Note that $ f_2(1) = 0 $ and 
\begin{equation*}
	f_2\left( 
	\left(\frac{1-\lambda}{\lambda}\right)^2 \right)=\frac{(1-\lambda ) \left(4 
		\lambda +\log \left(\frac{1-\lambda }{\lambda 
		}\right)-2\right)}{\lambda } > 0.
\end{equation*}
To see the above inequality, we define $ g_3(\lambda) = 4 \lambda +\log 
\left(\frac{1-\lambda }{\lambda }\right)-2 $. Its derivative is $ g'_3(\lambda) 
= 
\frac{(1-2 \lambda )^2}{(\lambda -1) \lambda } <0 $, which implies that $ 
g_3(\lambda) $ strictly decreases and that $ g_3(\lambda) > g_3(1/2) = 0 $ for 
$ \lambda\in (0,1/2) $. As a result, we deduce that $ f_2\left( 
\left(\frac{1-\lambda}{\lambda}\right)^2 \right) >0 $.
Thus we obtain that the function $ f_2(t) $ has a unique root (denoted by $ 
\rho_2(\lambda) $) on $ I $ and that $ f'_1(t) $ is positive on $ I_5 = 
(1,\rho_2(\lambda)) $ and negative on $ I_6 = \left(\rho_2(\lambda), 
\left( 
\frac{1-\lambda}{\lambda} 
\right)^2 \right) $, which implies that $ f_1 $ strictly increases on $ I_5 $ 
and strictly decreases on $ I_6 $. Note that $ f_1(1) = f_1\left( 
\left(\frac{1-\lambda}{\lambda}\right)^2 \right) =  0 $. We conclude that $ 
f_1(t) > 0 $ on $ I $, which implies that $ \kappa'_\lambda(t) > 0  $ on $ I $.

From the above analysis, we see that if $ \lambda\in (0,1/2) $, the function $ 
\kappa'_\lambda(t) $ has no real root on $ (0,\infty) \setminus \{ 
1,\left(\frac{1-\lambda}{\lambda}\right)^2 \} $.
Since \begin{equation*}
	\lim_{t\to 1}\kappa_\lambda(t ) = 4 (1-\lambda ) \lambda>0, \quad 
	\kappa_\lambda\left( 
	\left(\frac{1-\lambda}{\lambda}\right)^2 \right) = 0,
\end{equation*}
	we deduce that the derivative $ \kappa'_\lambda(t)  $ has a unique root at 
	$ 
	t=\left(\frac{1-\lambda }{\lambda }\right)^2 $ if $ \lambda\in (0,1/2) $. 
	By \eqref{eq:property_of_kappa}, we know that it also holds for $ \lambda\in 
	(1/2,1) $. Furthermore, we know that 
	the derivative is positive 
	if $ t<\left(\frac{1-\lambda }{\lambda }\right)^2 $ and is negative if $ 
	t>\left(\frac{1-\lambda }{\lambda }\right)^2 $. 
	Thus the maximum of $ 
	\kappa_\lambda $ is attained at $ t=\left(\frac{1-\lambda }{\lambda 
	}\right)^2 $ and it is exactly $ U(\lambda) $. 

Next, we assume $ \lambda=1/2 $. We have \begin{equation*}
\kappa_{1/2}(t) = \frac{t \log (t)+(t+1) (\log (2)-\log 
(t+1))}{\left(\sqrt{t}-1\right)^2}.
\end{equation*}
Its derivative is \begin{equation*}
	\kappa'_{1/2}(t) = \frac{\left(\sqrt{t}+1\right) \log 
	\left(\frac{t+1}{2}\right)-\sqrt{t} \log (t)}{\left(\sqrt{t}-1\right)^3 
	\sqrt{t}}
\end{equation*}
Define $ f_5(t) = \left(\sqrt{t}+1\right) \log 
\left(\frac{t+1}{2}\right)-\sqrt{t} \log (t) $. Its derivative is 
\begin{equation*}
	f'_5(t) = \frac{2 \left(\sqrt{t}-1\right)+(t+1) \log 
	\left(\frac{t+1}{2}\right)-(t+1) \log (t)}{2 \sqrt{t} (t+1)}.
\end{equation*}
Then we define $ f_6(t) =2 \left(\sqrt{t}-1\right)+(t+1) \log 
\left(\frac{t+1}{2}\right)-(t+1) \log (t)$, whose derivative is 
\begin{equation*}
	f'_6(t) = \frac{\sqrt{t}-1}{t}-\log (2 t)+\log (t+1)
\end{equation*}
and second derivative \begin{equation*}
	f''_6(t) = \frac{1}{t^3+t^2}-\frac{1}{2 t^{3/2}}.
\end{equation*}
If we set $ f''_6(t) >0 $, we get $ t^{1/2}+t^{3/2} <2 $, which is equivalent 
to $ t<1 $. Therefore $ f''_6(t) $ is positive on $ (0,1) $ and negative on $ 
(1,\infty) $, which implies that $ f'_6(t) < f'_6(1) = 0 $ for $ t\ne 1 $. We 
deduce that $ 
f_6(t) $ is strictly decreasing in $ t $ and thus has a unique root. Since $ 
t=1 $ is a root of $ f_6(t) $, it is the unique root, which implies that $ 
f_6(t) $ and $ f'_5(t) $ are both positive on $ (0,1) $ and negative on $ 
(1,\infty) $. As a result, we deduce that $ f_5(t) < f_5(1) = 0 $ for $ t\ne 1 
$. Thus we conclude that $ \kappa'_{1/2}(t) $ is positive on $ (0,1) $ and 
negative on $ (1,\infty) $. We can verify that $ t=1 $ is indeed a root of $ 
\kappa'_{1/2}(t) $.

So far we have shown for $ t\in (0,1) $ that 
the derivative $ \kappa'_{\lambda}(t) $ is positive 
if $ t<\left(\frac{1-\lambda }{\lambda }\right)^2 $ and is negative if $ 
t>\left(\frac{1-\lambda }{\lambda }\right)^2 $. 
Thus the maximum of $ 
\kappa_\lambda $ is attained at $ t=\left(\frac{1-\lambda }{\lambda 
}\right)^2 $ and it is exactly $ U(\lambda) $.

The infimum is 
\begin{equation*}
\begin{split}
&\min\{ \lim_{t\to 0^+} \kappa_\lambda(t), \lim_{t\to \infty} \kappa_\lambda(t) 
\\= & \min\{ -2(1-\lambda)\ln(1-\lambda), -2\lambda\ln \lambda \}.
\}
\end{split}
\end{equation*}
Therefore we conclude $ \kappa_\lambda\in [L(\lambda), U(\lambda)] $.

\section{Proof of \cref{thm:approximation}}\label{app:approximation}

In addition to \cref{lem:bound}, we need the following lemma.
\begin{lemma}[Theorem 6 of \citep{sason2016f}]\label{thm:sason}
	Let $ f $ and $ g $ be two convex functions that satisfy $ f(1)=0 $ and $ 
	g(1)=0 $, respectively. The function $ g(t)>0 $ for every $ t\in 
	(0,1)\cup(1,\infty) $. Let $ P $ and $ Q $ be two distributions on a common 
	finite sample space $ \Omega $. Define $ \beta_1 = 
	\inf_{i\in \Omega} \frac{Q(i)}{P(i)} $ and $ \beta_2 = \inf_{i\in \Omega} 
	\frac{P(i)}{Q(i)} $. We assume 
	that $ \beta_1,\beta_2\in [0,1) $. Then we have \begin{equation*}
	D_f(P\parallel Q)\le \kappa^* D_g(P\parallel Q),
	\end{equation*}
	where \begin{equation*}
	\kappa^* = \sup_{\beta\in (\beta_2, 1)\cup (1,\beta_1^{-1})} 
	\frac{f(\beta)}{g(\beta)}.
	\end{equation*}
\end{lemma}
By \cref{thm:sason,lem:bound}, we have
\begin{equation*}
L(\lambda)H^2(P, Q)\le	\wjs{\lambda}{P}{Q} \le U(\lambda)H^2(P, Q).
\end{equation*}

Now we show that $ U(\lambda)\le 1 $. Its derivative $ U'(\lambda) $ has a 
unique root at $ \lambda=1/2 $ on the interval $ (0,1) $ and it is positive if 
$ \lambda <1/2 $ and negative if $ \lambda > 1/2 $. Therefore $ U(\lambda)\le 
U(1/2) =1 $.

\section{Proof of 
\cref{lem:wjs-as-f-divergence}}\label{app:wjs-as-f-divergence}
	The equation $ m_\lambda(1)=0 $ can be verified by plugging in $ t=1 $ 
	directly. We compute the second derivative of $ m_\lambda $
	\begin{equation*}
	\frac{d^2 m_\lambda}{d t^2} = 
	\frac{\lambda(1-\lambda)}{t^2\lambda+(1-\lambda)t}.
	\end{equation*}
	If $ \lambda\in [0,1] $ and $ t\in (0, \infty) $, we have $ \frac{d^2 
		m_\lambda}{d t^2}\geq 0 $, which implies the convexity of $ m_\lambda $.
	
	The $ m_\lambda $-divergence equals to 
	\begin{equation*}
		D_{m_\lambda}(P\parallel Q) = \int_{\Omega} \lambda \ln \frac{dP}{dQ}dP
		-(\lambda dP+(1-\lambda)dQ)\ln\left(\lambda 
		\frac{dP}{dQ}+1-\lambda\right)
	\end{equation*}
	while the MIL-divergence equals
	\begin{align*}
		\wjs{\lambda}{P}{Q}
		={} & \int_{\Omega} \lambda\ln\frac{dP/dQ}{\lambda 
			dP/dQ+(1-\lambda)}dP + (1-\lambda)\ln\frac{1}{\lambda 
			dP/dQ+(1-\lambda)}dQ\\
		={} & \int_{\Omega} \lambda \ln \frac{dP}{dQ}dP - (\lambda dP + 
		(1-\lambda)dQ)\ln 
		\left( 
		\lambda 
		\frac{dP}{dQ} + 1-\lambda \right)\,.
	\end{align*}
	Thus we conclude that the $ m_\lambda $-divergence yields the 
	MIL-divergence 
	with 
	parameter $ \lambda $.

\section{Proof of \cref{thm:main}}\label{app:main}
Let $ P $ and $ Q $ be two probability measures in $ \mathcal{P} $. If $ P $ 
and $ Q $ are equal, $ D_f(P\parallel Q) = 0 $. Therefore for any hash function 
$ h $, it holds that $ h(P) = h(Q) $, which implies that $ \Pr_{h\sim 
\mathcal{H}}[h(P)=h(Q)]=1\ge p_1 $.

In the sequel, we assume that $ P $ and $ Q $ are different.
Since $ P $ and $ Q $ are two different distributions, there exists $ i\in 
\Omega $ such that $ P(i) < Q(i) $. We show this by contradiction. Assume that 
$ \forall i\in \Omega $, $ P(i)\ge Q(i) $. Since $ P $ and $ Q $ are different, 
there exists $ i_0\in \Omega $ such that $ P(i_0) \ne Q(i_0) $. Since $ P(i)\ge 
Q(i) $ holds for $ \forall i\in \Omega $, we have $ P(i_0) > Q(i_0) $. 
Therefore $ \sum_{i\in \Omega} P(i) > \sum_{i\in \Omega} Q(i) $. However, both 
$ P $ and $ Q $ sum to $ 1 $, which leads to a contradiction. Therefore, we 
obtain the existence of $ i $ such that $ P(i) < Q(i) $, which yields $ 
\beta_2\triangleq \inf_{i\in \Omega}\frac{P(i)}{Q(i)} <1 $. Similarly, we have 
$ \beta_1\triangleq \inf_{i\in \Omega}\frac{Q(i)}{P(i)} <1 $. Since $ P(i) $ 
and $ Q(i) $ are 
non-negative for $ \forall i\in \Omega $, we have $ \beta_1,\beta_2\ge 0 $. In 
sum, we showed that $ \beta_1, \beta_2\in [0,1) $. By the definition of $ 
\beta_0 $, we know the following interval inclusion
\begin{equation*}
(\beta_2, \beta_1^{-1}) \subseteq (\beta_0, \beta_0^{-1}).
\end{equation*}

Recall that
\begin{equation*}
\begin{split}
U=&\sup_{\beta\in (\beta_0, 1)\cup (1,\beta_0^{-1})} \frac{f(\beta)}{g(\beta)}, 
\\ 
L =& \inf_{\beta\in (\beta_0, 1)\cup (1,\beta_0^{-1})} 
\frac{f(\beta)}{g(\beta)}.
\end{split}
\end{equation*}
By \cref{thm:sason}, we obtain the approximation guarantee 
\begin{equation}\label{eq:general_twosided_approximation}
L\cdot D_g(P\parallel
Q) \le  D_f(P\parallel Q) \le U\cdot D_g(P\parallel
 Q)
\end{equation}

There are two cases to consider. In the first case, we assume that $ 
D_f(P\parallel 
Q)\le L r_1 $. By \eqref{eq:general_twosided_approximation}, we have $ 
D_g(P\parallel Q) \le r_1 $. Since $ \mathcal{H} $ is an $ (r_1, r_2, p_1,p_2) 
$-sensitive family for $ g 
$-divergence, it holds that $ \pr_{h\sim 
	\mathcal{H}}[h(P)=h(Q)]\ge p_1 $. Similarly, if $ D_g(P\parallel Q) > Ur_2 
	$, we have $ \pr_{h\sim 
	\mathcal{H}}[h(P)=h(Q)]\le p_2 $. Thus, $ \mathcal{H} $ forms an $ (Lr_1, 
	Ur_2, p_1,p_2) $-sensitive family for $ f 
$-divergence on $ \mathcal{P} $.

\section{Proof of \cref{thm:lsh-family}}\label{sub:proof_wjs}
If $ \wjs{\lambda}{P}{Q}\le R $, by \cref{thm:approximation}, we have 
\begin{equation*}
	\left\| \sqrt{P}-\sqrt{Q} \right\|_2 \le 
	\sqrt{\frac{2R}{L(\lambda)}}\triangleq R_1.
\end{equation*}
If $ \wjs{\lambda}{P}{Q}\ge c^2\frac{U(\lambda)}{L(\lambda)}R $, we have 
\begin{equation*}
\left\| \sqrt{P}-\sqrt{Q} \right\|_2 \ge
c\sqrt{\frac{2R}{L(\lambda)}}= cR_1.
\end{equation*}
By the construction and properties of locality-sensitive hash family for $ L^2 
$ distance proposed in \cite[Section~3.2]{datar2004locality}, we know that $ 
h_{\mathbf{a},b} 
$ 
forms a $ (R_1,cR_1,p_1,p_2) $-sensitive hash family for the $ L^2 $ distance 
between two vectors $ \sqrt{P} $ and $ \sqrt{Q} $. Therefore, provided that $ 
\wjs{\lambda}{P}{Q}\le R $, which implies $ \left\| 
\sqrt{P}-\sqrt{Q} \right\|_2\le R_1 $, we have 
\begin{equation*}
	\pr[h_{\mathbf{a},b}(P)=h_{\mathbf{a},b}(Q)] \ge p_1.
\end{equation*}
Similarly, if $ \wjs{\lambda}{P}{Q}\ge c^2\frac{U(\lambda)}{L(\lambda)}R $, we 
have 
\begin{equation*}
	\pr[h_{\mathbf{a},b}(P)=h_{\mathbf{a},b}(Q)] \le p_2.
\end{equation*}

\section{Proof of \cref{thm:lsh_family_triangular_discrimination}}
\label{sub:proof_triangular_discrimination}
The derivative of the ratio function $ \kappa(t) = \frac{\delta(t)}{\hel(t)} $ 
is 
\begin{equation*}
	\kappa'(t) = \frac{1-t}{\sqrt{t} (t+1)^2}.
\end{equation*}
It is positive when $ t<1 $ and negative when $ t>1 $. Therefore for $ \forall 
t\in (0,\infty) $, $ \kappa(t)\le 
\kappa(1) = 2 $ and \begin{equation*}
\kappa(t)\ge \min\{ \lim_{t\to 0^+}\kappa(t),\lim_{t\to \infty}\kappa(t) \} = 1.
\end{equation*}
By \cref{thm:sason}, we have \begin{equation*}
	H^2(P,Q)\le \td{P}{Q}\le 2H^2(P,Q).
\end{equation*}

If $ \td{P}{Q}\le R $,  we have 
\begin{equation*}
\left\| \sqrt{P}-\sqrt{Q} \right\|_2 \le 
\sqrt{2R}\triangleq R_1.
\end{equation*}
If $ \wjs{\lambda}{P}{Q}\ge 2c^2R $, we have 
\begin{equation*}
\left\| \sqrt{P}-\sqrt{Q} \right\|_2 \ge
\sqrt{2R}c= cR_1.
\end{equation*}
By the construction and properties of locality-sensitive hash family for $ L^2 
$ distance proposed in \cite[Section~3.2]{datar2004locality}, we know that $ 
h_{\mathbf{a},b} 
$ 
forms a $ (R_1,cR_1,p_1,p_2) $-sensitive hash family for the $ L^2 $ distance 
between two vectors $ \sqrt{P} $ and $ \sqrt{Q} $. Therefore, provided that $ 
\td{P}{Q}\le R $, which implies $ \left\| 
\sqrt{P}-\sqrt{Q} \right\|_2\le R_1 $, we have 
\begin{equation*}
\pr[h_{\mathbf{a},b}(P)=h_{\mathbf{a},b}(Q)] \ge p_1.
\end{equation*}
Similarly, if $ \td{P}{Q}\ge 2c^2R $, we 
have 
\begin{equation*}
\pr[h_{\mathbf{a},b}(P)=h_{\mathbf{a},b}(Q)] \le p_2.
\end{equation*}

\section{Proof of \cref{lem:pd_k}}\label{app:pd_k}

	First, we would like to note that $ k $ is homogeneous, \ie,
for all $c\ge 0$, it holds that $k(cx,cy)=ck(x,y)$.
Its kernel signature~\citep{vedaldi2012efficient} is
\begin{equation*}
\cK(\lambda)\triangleq k(e^{\lambda/2},e^{-\lambda/2}) = e^{-\frac{\lambda 
	}{2}} \left(\left(e^{\lambda }+1\right) \ln \left(e^{\lambda 
}+1\right)-e^{\lambda } \lambda \right)\,.
\end{equation*}
First, let us review the definition of a positive definite function.
\begin{definition}[\citep{bochner1959lectures}]
	We call a complex-valued function $f:\bR\to \bC$ is positive definite if
	\begin{compactenum}
		\item it is continuous in the finite region and is bounded on $\bR$
		\item it is Hermitian, \ie, $\overline{f(-x)}=f(x)$
		\item it satisfies the following conditions: for any real numbers 
		$x_1,\dots,x_n\in \bR$, the matrix \[
		A=(f(x_i-x_j))_{i,j=1}^n
		\]
		is positive semidefinite.
	\end{compactenum}
\end{definition}
Next we will show that $ \cK $ is a positive definite function by showing 
that it is the Fourier transform of a non-negative function.
We have the following Fourier transform and inverse Fourier transform
\begin{align*}
\cK(\lambda) ={}& \int_\bR e^{-i\lambda w}\frac{2\sech(\pi 
	w)}{1+4w^2}dw\,,\\
\kappa(w)\triangleq{} & \frac{1}{2\pi}\int_\bR \cK(\lambda)e^{i\lambda 
	w}d\lambda=\frac{2\sech(\pi w)}{1+4w^2}\,.\\
\end{align*}
Then we need the following 
lemmata.
\begin{lemma}\label{lem:pd}
	If $f(x)=\int_\bR e^{-ixt}g(t)dt$ is the Fourier transform of a 
	non-negative function $g(t)$, then it is positive definite.
\end{lemma}
\begin{proof}[Proof of \cref{lem:pd}]
	Let $x_1,\dots,x_n\in \bR$ be arbitrary real numbers and 
	$a_1,\dots,a_n$ be 
	arbitrary complex numbers. Let us compute the quadratic form directly
	\begin{align*}
	\sum_{j,k=1}^n f(x_j-x_k)a_j\overline{a_k} = \int_\bR \sum_{j,k=1}^n 
	e^{-i(x_j-x_k)t}a_j\overline{a_k}g(t)dt = \int_\bR \left| \sum_{j=1}^n 
	a_je^{-ix_j t} \right|^2 g(t)dt \ge 0\,.
	\end{align*}
\end{proof}

\begin{lemma}[Lemma~1 in \citep{vedaldi2012efficient}]\label{lem:pdkernel}
	A homogeneous kernel is positive definite if, and only if,
	its signature $\cK(\lambda)$ is a positive definite function. 
\end{lemma}

Since $\frac{2\sech(\pi w)}{1+4w^2}\ge 0$ holds for $\forall w\in \bR$, we 
deduce that $\cK(\lambda)$ is the Fourier transform of a non-negative function. 
\cref{lem:pd} implies that $\cK(\lambda)$ is a positive definite function. 
Therefore $k$ is a positive definite kernel by \cref{lem:pdkernel}.

Let us define the feature map \[
\Phi_w(x) \triangleq e^{-iw\ln(x)}\sqrt{x\frac{2\sech(\pi w)}{1+4w^2}}\,.
\]
Since $k(x,y)$ is homogeneous, we have
\begin{align*}
k(x,y)={} & \sqrt{xy}k(\sqrt{x/y},\sqrt{y/x}) = \sqrt{xy} \cK(\ln(y/x))\\
={} & \sqrt{xy}\int_{\bR} e^{-i\ln(y/x)w} \frac{2\sech(\pi w)}{1+4w^2}dw
= \int_\bR \Phi_w(x)^* \Phi_w(y) dw\,.
\end{align*}

\section{Proof of \cref{thm:krein_kernel}}\label{app:krein_kernel}
Let $ 
z $ denote the merged value. If we define $\eta(u)\triangleq -u\ln(u)$,
the mutual information loss is
\begin{align*}
\mil(\bfx,\bfy) ={} & \sum_{c\in \cC}\left[ p(c,x)\ln 
\frac{p(c,x)}{p(c)p(x)} + 
p(c,y)\ln \frac{p(c,y)}{p(c)p(y)} 
- p(c,z)\ln \frac{p(c,z)}{p(c)p(z)} \right] \\
={} & \sum_{c\in \cC}\left[ p(c,x)\ln 
\frac{p(c,x)}{p(x)} + 
p(c,y)\ln \frac{p(c,y)}{p(y)} 
- p(c,z)\ln \frac{p(c,z)}{p(z)} \right]\\
={} & \eta(p(x)) 
+ \eta(p(y)) - \eta(p(z)) - \sum_{c\in \cC} \left[ 
\eta(p(c,x))+\eta(p(c,y))-\eta(p(c,z)) \right] \,.
\end{align*}
By the definition of $ k $, we have \[ 
k(a,b) = \eta(a)+\eta(b)-\eta(a+b)\enspace.
\]
As a result, we re-write $ \mil(\bfx,\bfy) $ as
\[ 
\mil(\bfx,\bfy) 
= k(p(x),p(y))- \sum_{c\in \cC} k(p(c,x),p(c,y))
= K_1(\bfx,\bfy)-K_2(\bfx,\bfy)\enspace.
\]

\cref{lem:pd_k} indicates that $ k $ is a positive definite kernel. 
In light of the techniques for constructing new kernels presented in \citep[Section~6.2]{bishop2006pattern}, we obtain that 
that $ K_1 $ and $ K_2 $ are positive definite kernels.

\section{Proof of \cref{lem:truncation}}\label{app:truncation}
Recall that $ k(x,y) = \int_\bR \Phi_w(x)^*\Phi_w(y)dw $. We have
\begin{align*}
\left| k(x,y) - \int_{-t}^{t} \Phi_w(x)^* \Phi_w(y)  dw \right| ={}& 
\left| \int_{|w|>t} \Phi_w(x)^*\Phi_w(y)dw \right|
\le \int_{|w|>t} \left| e^{iw\ln(x/y)}\sqrt{xy}\rho(w) \right| dw\\
\stackrel{(a)}{\le}& 
2\int_{t}^\infty 
\rho(w)dw
\stackrel{(b)}{\le}  8\int_{t}^\infty e^{-\pi w}dw = \frac{8}{\pi} e^{-\pi 
	t} \le 
4e^{-t}\enspace,
\end{align*}
where $ (a) $ is due to $ \left| e^{iw\ln(x/y)}\sqrt{xy} \right| \le 1 $
and $ (b) $ is due to \[ 
\frac{2\sech(\pi w)}{1+4w^2}\le 2\sech(\pi w) = \frac{4}{e^{\pi w}+e^{-\pi w}} \le 4e^{-\pi w}\enspace.
\]
\section{Proof of \cref{lem:discretization}}\label{app:discretization}

As the first step, we re-write the integral\[ 
\int_{-\Delta J}^{\Delta J} \Phi_w(x)^* \Phi_w(y)  dw = \sum_{j=-J+1}^J \int_{(j-1)\Delta}^{j\Delta} e^{iw\ln(x/y)}\sqrt{xy}\rho(w)  dw\enspace.
\]
Then we bound the discretization error \begin{align*}
& \left| \int_{-\Delta J}^{\Delta J} \Phi_w(x)^* \Phi_w(y)  dw - 
\sum_{j=-J+1}^J \int_{(j-1)\Delta}^{j\Delta} e^{iw_j\ln(x/y)}\sqrt{xy}\rho(w)  dw
\right|\\
\le{} & \sum_{j=-J+1}^J \int_{(j-1)\Delta}^{j\Delta} \left| e^{iw\ln(x/y)} - e^{iw_j\ln(x/y)} \right| \sqrt{xy}\rho(w)  dw \\
\stackrel{(a)}{\le} & \sum_{j=-J+1}^J \int_{(j-1)\Delta}^{j\Delta} |\ln(x/y)|\frac{\Delta}{2} \sqrt{xy}\rho(w)  dw
= \frac{\Delta}{2}\sqrt{xy}|\ln(x/y)|\int_{-\Delta J}^{\Delta J} \rho(w)dw
\stackrel{(b)}{\le} 2\Delta %
\enspace,
\end{align*}
where $ (a) $ is due to \[ 
\left| e^{iw\ln(x/y)} - e^{iw_j\ln(x/y)} \right| \le |\ln(x/y)| |w-w_j| \le \frac{\Delta}{2}|\ln(x/y)|\enspace.
\] and $ (b) $ is due to $ \int_{-\Delta J}^{\Delta J} \rho(w)dw \le \int_\bR \rho(w)dw = 2\ln 2 $ and $ \sqrt{xy}|\ln(x/y)|\le \sqrt{x}|\ln(x)|+\sqrt{y}|\ln(y)|\le \frac{4}{e} $.

Next we re-write the partial Riemann sum by substituting the new index $ k=1-j $ \begin{align*}
& \sum_{j=-J+1}^0 \int_{(j-1)\Delta}^{j\Delta} e^{iw_j\ln(x/y)}\sqrt{xy}\rho(w)  dw
= \sum_{k=1}^J \int_{-k\Delta}^{(1-k)\Delta} e^{i(1/2-k)\Delta \ln(x/y)}\sqrt{xy}\rho(w)dw\\
={}& \sum_{k=1}^J \int_{(k-1)\Delta}^{k\Delta} e^{-iw_k\ln(x/y)} \sqrt{xy} \rho(w)dw\enspace.
\end{align*}
Therefore the entire Riemann sum can be re-written as
\begin{align*} 
& \sum_{j=-J+1}^J \int_{(j-1)\Delta}^{j\Delta} e^{iw_j\ln(x/y)}\sqrt{xy}\rho(w)  dw 
= \sum_{j=1}^J \int_{(j-1)\Delta}^{j\Delta} (e^{iw_j\ln(x/y)} +e^{-iw_j\ln(x/y)}) \sqrt{xy}\rho(w)  dw\\
={} & 2\sum_{j=1}^J  (\cos(w_j\ln x)\cos(w_j\ln y) + \sin(w_j\ln x)\sin(w_j\ln y)) \sqrt{xy}\int_{(j-1)\Delta}^{j\Delta}\rho(w)  dw\\
={}& \left\langle \bigoplus_{j=1}^J \tau(x,w_j,j), \bigoplus_{j=1}^J \tau(y,w_j,j)\right\rangle\enspace.
\end{align*}

\end{appendices}

\end{document}